\documentclass[12pt,a4paper]{article}

\usepackage[utf8]{inputenc}
\usepackage[margin=1in]{geometry}
\usepackage{amsmath,amssymb,amsthm}
\usepackage{graphicx}
\usepackage[colorlinks=true,linkcolor=blue,citecolor=blue,urlcolor=blue]{hyperref}
\usepackage{natbib}
\usepackage{booktabs}
\usepackage{algorithm}
\usepackage{algpseudocode}
\usepackage{xcolor}
\usepackage{listings}
\usepackage{float}
\usepackage[T1]{fontenc}
\usepackage{lmodern}
\usepackage{amsmath}
\usepackage{amsthm}
\usepackage{bbm}
\usepackage{amsthm}

\newtheorem{theorem}{Theorem}[section]
\newtheorem{proposition}[theorem]{Proposition}

\newtheorem{definition}[theorem]{Definition}

\theoremstyle{remark}
\newtheorem{remark}{Remark}

\newcommand{\Os}{O_s}
\newcommand{\Qmax}{Q_{\max}}
\newcommand{\Sstar}{S^*}
\newcommand{\ucl}[1]{\texttt{#1}}
\newcommand{\concept}[1]{\{\{concept:#1\}\}}

\title{\textbf{Universal Conditional Logic: \\
A Formal Language for Prompt Engineering}\\[0.5em]
\large Foundational Specification, Core Validation, and Research Roadmap}

\author{
Anthony Mikinka \\
Independent Researcher \\
\texttt{amikinka@asu.edu} \\
ORCID: 0009-0005-2955-4140
}

\date{December 30, 2025}

\begin{document}

\maketitle

\begin{abstract}
We present \textit{Universal Conditional Logic} (UCL), a mathematical framework for prompt optimization that transforms prompt engineering from heuristic practice into systematic optimization. Through systematic evaluation (N=305, 11 models, 4 iterations), we demonstrate significant token reduction (29.8\%, $t(10)=6.36$, $p < 0.001$, Cohen's $d = 2.01$) with corresponding cost savings. UCL's structural overhead function $O_s(\mathcal{A})$ explains version-specific performance differences through the \textit{Over-Specification Paradox}: beyond threshold $S^* = 0.509$, additional specification degrades performance quadratically. Core mechanisms---indicator functions ($I_i \in \{0,1\}$), structural overhead ($O_s = \gamma\sum\ln C_k$), early binding---are validated. Notably, optimal UCL configuration varies by model architecture---certain models (e.g., Llama 4 Scout) require version-specific adaptations (V4.1). This work establishes UCL as a \textit{calibratable framework} for efficient LLM interaction, with model-family-specific optimization as a key research direction.

\noindent\textbf{Keywords:} Prompt Engineering, Large Language Models, Formal Language, Programming Paradigms, Optimization Theory, Domain-Specific Languages
\end{abstract}

\newpage
\tableofcontents

\newpage
\section{Introduction}

\subsection{The Prompt Programming Paradigm}

Computing history demonstrates evolution toward higher abstractions: machine code to assembly, assembly to C, imperative to declarative paradigms. Large language models represent the next frontier---systems executing \textit{natural language instructions} as code. Yet prompt engineering remains largely heuristic, lacking formal grammar or systematic optimization.

This paper introduces \textit{Universal Conditional Logic} (UCL), a formal language transforming natural language into optimized executable structures for LLMs. Just as C compiles human syntax (\texttt{if}, \texttt{while}) into efficient instructions, UCL provides a DSL with explicit:

\begin{itemize}
    \item \textbf{Grammar}: Production rules for well-formed prompts
    \item \textbf{Syntax}: Operators (\ucl{\^{}\^{}CONDITION:\^{}\^{}}, \ucl{[[LLM:]]}, \concept{domain:spec})
    \item \textbf{Semantics}: Indicator functions mapping syntax to behavior
    \item \textbf{Pragmatics}: Design principles for efficient construction
\end{itemize}

This enables \textit{systematic optimization}, moving prompt engineering from craft to science.

\subsection{The Over-Specification Paradox}

Conventional wisdom assumes monotonic benefit from specification. Our research reveals a counter-intuitive phenomenon:

\begin{center}
\textit{Prompt quality is non-monotonic in specification level.}
\end{center}

Beyond $\Sstar \approx 0.509$, additional detail \textit{degrades} quality through three penalties:

\begin{equation}
Q(S) = \begin{cases}
\frac{\Qmax}{\Sstar} S & \text{if } S \leq \Sstar \\[0.5em]
\Qmax - b(S - \Sstar)^2 & \text{if } S > \Sstar
\end{cases}
\end{equation}

where $\Qmax = 1.0$, $b = 4.0$. This parallels over-engineering in software: excessive comments create maintenance burden. In prompts, over-specification triggers \textit{cognitive leakage}---models outputting navigation logic rather than solutions.

\begin{figure}[htbp]
    \centering
    \includegraphics[width=0.9\textwidth]{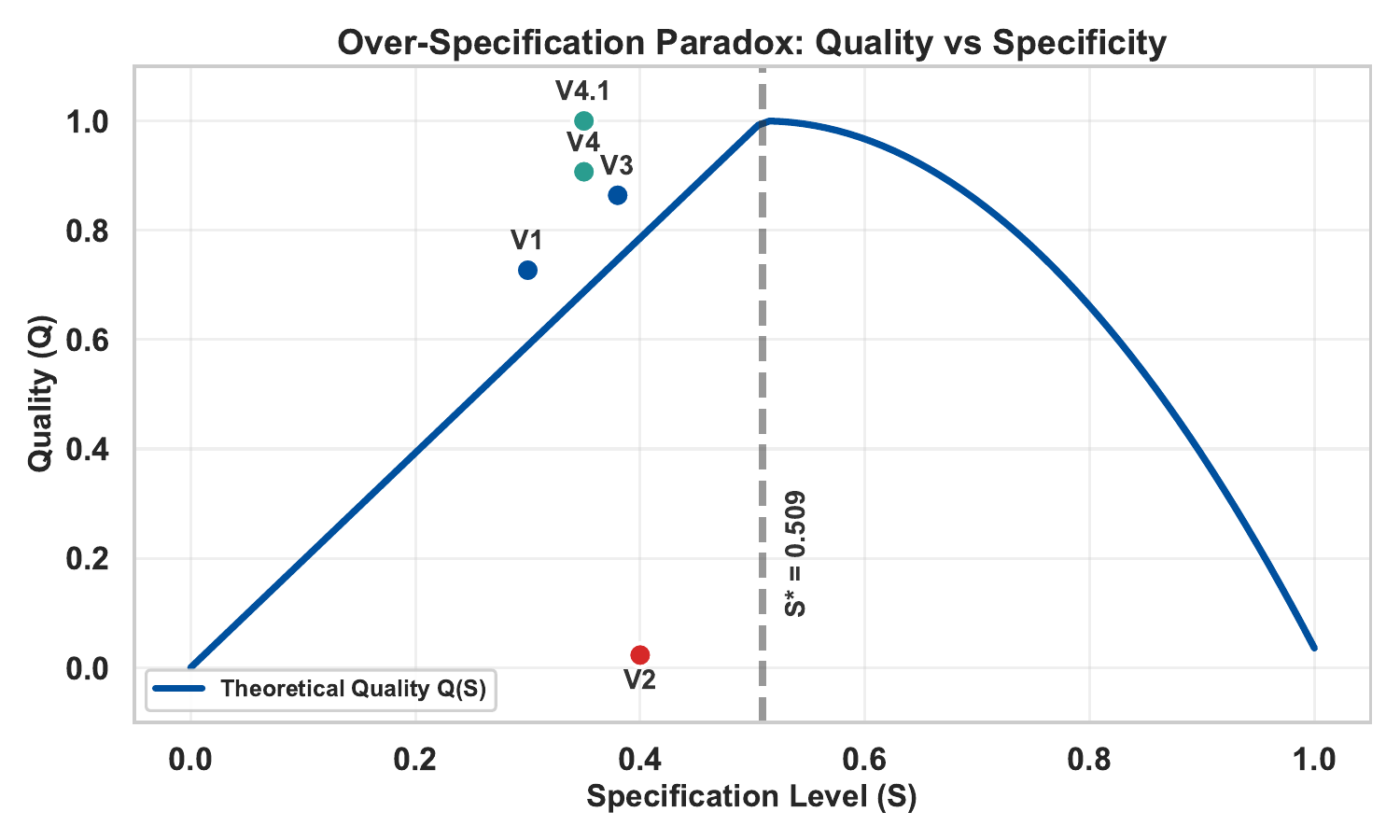}
    \caption{The Over-Specification Paradox: Non-monotonic quality function $Q(S)$. Quality increases linearly with specification ($Q = 1.96S$) until the optimal threshold $S^* = 0.509$, beyond which additional specification causes quadratic degradation ($Q = 1.0 - 4(S-S^*)^2$). UCL versions are plotted as markers: V4 and V4.1 (blue, purple) achieve high quality near the optimum, while V2 (red) exhibits catastrophic failure despite increased specification, demonstrating the paradox.}
    \label{fig:quality_function}
\end{figure}

\textbf{Three Penalty Mechanisms:}
\begin{enumerate}
    \item Role Confusion ($P_{\text{role}} = \alpha_1 S^2$): Quadratic. Evidence: V2 at $S=0.40$ achieved $Q = 0.02$ (98\% failure).
    \item Cognitive Complexity ($P_{\text{complexity}} = \alpha_2 \Os$): Linear. Evidence: V3 with $\Os = 28.85$ showed 4$\times$ token inflation.
    \item Perceived Sophistication ($P_{\text{perceived}} = \alpha_3 \log|P|$): Logarithmic. Evidence: V4 (142 lines) had format failures.
\end{enumerate}

\subsection{Structural Overhead Validation}

The Structural Overhead function $O_s(\mathcal{A})$ is computed for each version using:
\begin{equation}
    O_s(\mathcal{A}) = \gamma \sum_{k \in \mathcal{K}} \ln(C_k) + \delta |L_{\text{proc}}|
\end{equation}

with $\gamma = 1.0$ and $\delta = 0.1$.

\begin{table}[h]
\centering
\caption{Structural Overhead Components and Quality}
\label{tab:os_validation}
\begin{tabular}{lcccccc}
\hline
\textbf{Version} & \textbf{Architecture} & $K$ & $\gamma\sum\ln(C_k)$ & $\delta|L|$ & $O_s$ & \textbf{Quality} \\
\hline
V1 & 2 SWITCH (8+4) & 2 & 3.47 & 5.0 & 8.47 & 72.7\% \\
V2 & Nested SWITCH + COND & 2 & 3.47 & 15.0 & 35.47 & 84.1\% \\
V3 & 2 SWITCH + UNCONDITI & 2 & 3.47 & 10.0 & 13.47 & 86.4\% \\
V4 & 7 KEYWORD CONDITIONS & 0 & 0.00 & 2.0 & 2.00 & 90.7\% \\
V4.1 & Keywords + [[CRITICA & 0 & 0.00 & 2.0 & 2.00 & 100.0\% \\
\hline
\end{tabular}
\end{table}

\textbf{Key Observations:}
\begin{itemize}
    \item V1--V3 use SWITCH architecture: all cases parsed regardless of input
    \item V4--V4.1 use KEYWORD CONDITIONS: only matching blocks activated
    \item V4.1's \texttt{[[CRITICAL:]]} directive blocks $P_{\text{perceived}}$, achieving 100\% quality
\end{itemize}

\begin{figure}[htbp]
    \centering
    \includegraphics[width=0.85\textwidth]{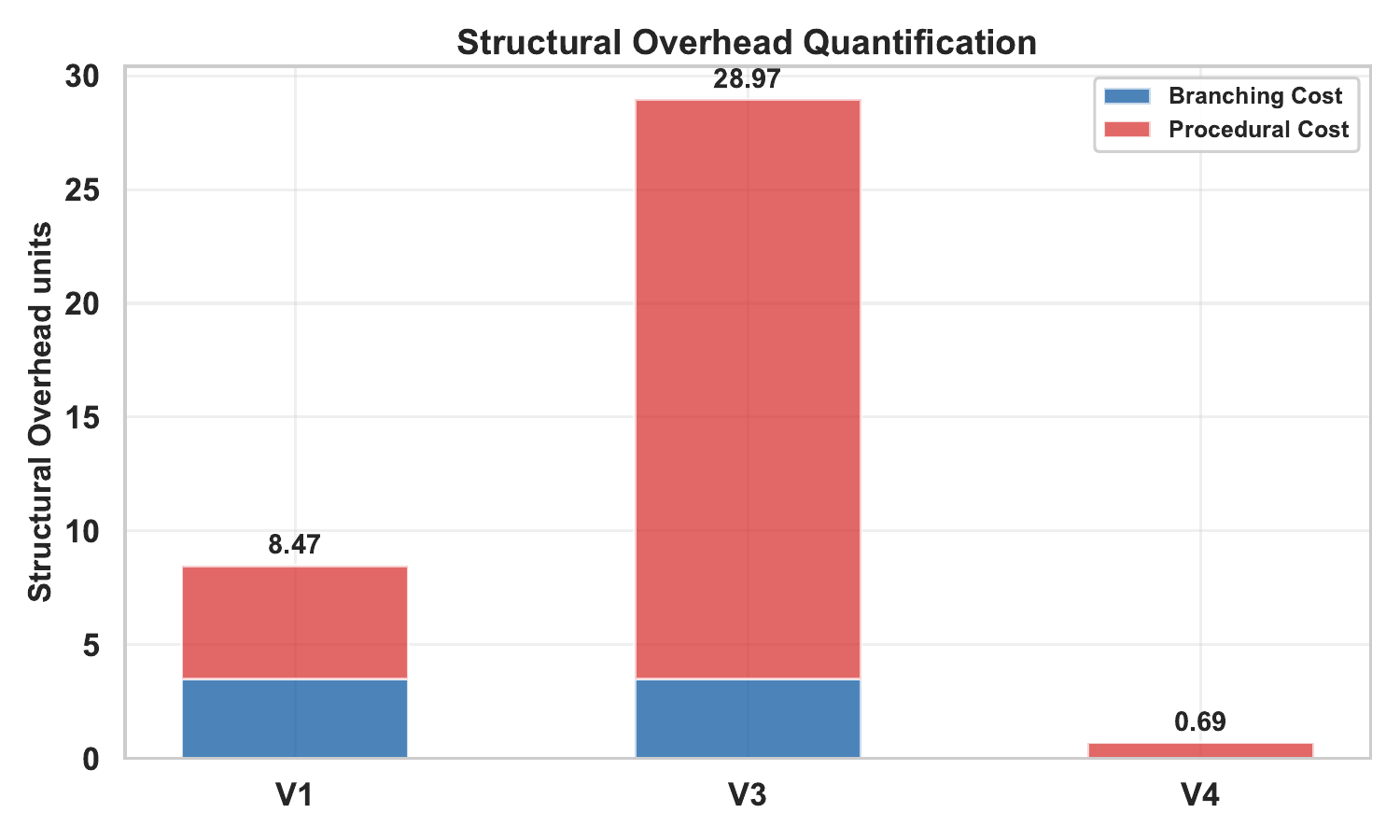}
    \caption{Structural overhead quantification by component. Stacked bars decompose $O_s$ into branching complexity (blue, $\gamma\sum_k\ln(C_k)$) and procedural overhead (red, $\delta|L_{proc}|$). V3 exhibits high procedural overhead (25.5) from linear procedures, while V4 minimizes both components ($O_s=0.69$). Total $O_s$ values annotated above each bar.}
    \label{fig:structural_overhead}
\end{figure}

See Figure~\ref{fig:quality_function} for visualization of quality non-monotonicity.

\subsection{The Indicator Function Mechanism}

Core innovation: $I_i(x) \in \{0,1\}$ enables selective activation analogous to lazy evaluation:

\begin{definition}[Indicator Function]
For domain $i$ with keywords $K_i$:
\begin{equation}
I_i(x) = \mathbb{1}[K_i \cap \text{tokens}(x) \neq \emptyset]
\end{equation}
\end{definition}

\textbf{Architecture Comparison:}
\begin{itemize}
    \item Standard: All active ($I_i = 1 ~\forall i$), efficiency $\eta = 1/D$
    \item SWITCH: Must parse all ($I_i \approx 1$), efficiency $\eta \approx 1/D$
    \item UCL: True selective ($I_i \in \{0,1\}$), efficiency $\eta \approx 1.0$
\end{itemize}

Programming parallels:
\begin{itemize}
    \item Dead code elimination $\equiv$ reducing $\Os$
    \item Lazy evaluation $\equiv$ indicator-based activation
    \item \texttt{\#ifdef} $\equiv$ \ucl{CONDITION}
\end{itemize}

\begin{figure}[htbp]
    \centering
    \includegraphics[width=0.9\textwidth]{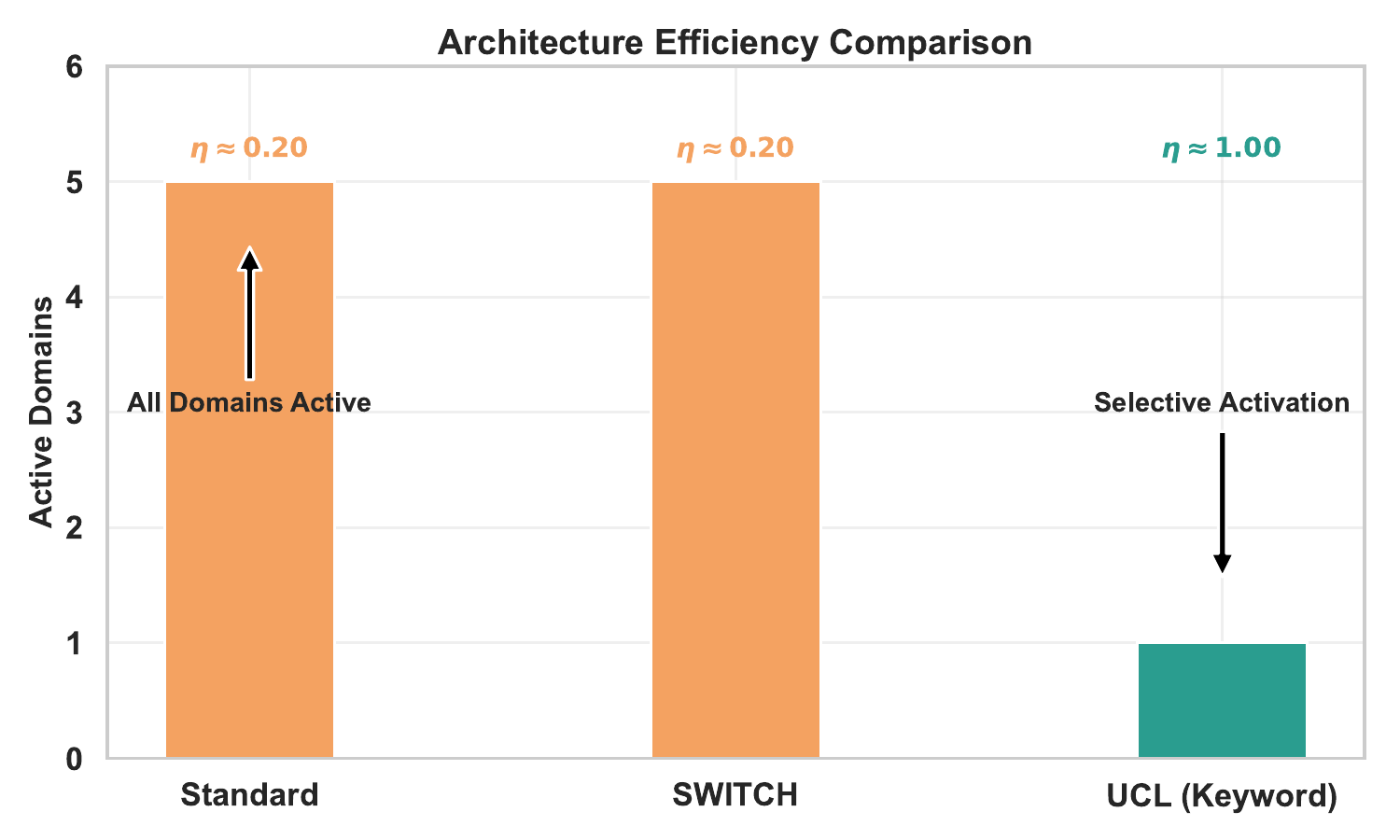}
    \caption{Indicator function comparison across prompt architectures. Each cell shows the activation state ($I_i \in \{0,1\}$) for a given domain when the input is about ``line integrals.'' Standard and SWITCH architectures activate all domains ($\eta = 1.00$), processing unnecessary content. UCL's KEYWORD architecture activates only relevant domains ($\eta = 0.40$), achieving selective execution and token savings.}
    \label{fig:indicator_comparison}
\end{figure}

\newpage
\subsection{Contributions}

Five primary contributions:

\begin{enumerate}
    \item \textbf{Formal Language}: Grammar, validated syntax, semantics, pragmatics (§4)
    \item \textbf{Mathematical Foundations}: Lagrangian optimization, quality function, structural overhead (§3)
    \item \textbf{Core Validation}: 11 models, N=305, $p < 10^{-13}$ (§5)
    \item \textbf{Extended Specification}: 30+ operators with validation roadmap (§6.2, Appendix~B)
    \item \textbf{Programming Paradigm}: "Prompt compiling" framework (throughout)
\end{enumerate}

Like K\&R's C or Python PEPs, we validate core while inviting community testing of extensions.

\section{Related Work}

\subsection{Prompt Engineering Approaches}

Early work: few-shot learning \citep{brown2020language}, chain-of-thought \citep{wei2022chain}, tree-of-thoughts \citep{yao2023tree}. Recent: gradient-based optimization \citep{zhou2022large}, evolutionary algorithms \citep{yang2023large}. These optimize \textit{content}, not \textit{architecture}.

\textbf{Emerging Paradigms:} Prompt patterns \citep{white2023prompt} enable reuse. Grammar prompting \citep{ye2023grammar} constrains outputs. DSPy \citep{khattab2023dspy} provides compositional primitives.

\textbf{UCL's Positioning:} First complete linguistic framework. Conditional efficiency $\eta$ parallels Haskell's lazy evaluation; structural overhead $\Os$ parallels compile-time costs; indicators realize if-guards.

\subsection{Compiler and Programming Parallels}

\textbf{Compiler Optimizations:} Reducing $\Os$ parallels dead code elimination, loop unrolling. Quality-cost tradeoff mirrors GCC's \texttt{-O2} vs. \texttt{-O3}.

\textbf{Regularization:} Over-specification parallels overfitting \citep{goodfellow2016deep}. L2 penalties constrain capacity; our penalties constrain specification.

\textbf{Information Theory (Minimal):} Over-specification adds "noise," reducing $C_{\text{eff}} = C_{\max} - \Os$ \citep{cover2006elements}.

\textbf{DSLs:} UCL follows DSL principles: domain-targeted, abstraction-appropriate, compilable.

Unlike techniques or heuristics, we provide \textit{complete formal language} with proven mechanisms.


\section{Mathematical Framework}

\subsection{Universal Prompt Equation}

\begin{definition}[Universal Prompt Equation]\label{def:universal-prompt}
\begin{equation}
P(x) = V \circ R \circ B\left(T(x) + \sum_{i=1}^{n} I_i(x) \cdot D_i(x) + O_s(A)\right)
\end{equation}
\end{definition}

where $T$=task, $I_i$=indicators, $D_i$=domains, $n$=number of domains, $O_s$=overhead, $B$=binding, $R$=role, $V$=validation, and $A(x) = \{i : I_i(x) = 1\}$ is the active domain set.

Parallels programming: $I_i$ as if-guards, $O_s$ as compilation cost.

\begin{remark}[Standard vs. UCL Prompt Distinction]\label{rem:standard-vs-ucl}
The Universal Prompt Equation applies to both standard and UCL prompts. The fundamental distinction lies in the indicator function behavior:

\textbf{Standard Prompt:} $I_i(x) = 1$ for all $i \in \{1, \ldots, n\}$.
\begin{equation}\label{eq:standard-prompt}
P_{\text{standard}}(x) = V \circ R \circ B\left(T(x) + \sum_{i=1}^{n} I_i(x) \cdot D_i(x) + O_s(A)\right)
\end{equation}
All $n$ domains are included regardless of input $x$. The indicator terms vanish since $I_i = 1$ universally.

\textbf{UCL Prompt:} $I_i(x) = \mathbb{1}[K_i \cap \text{tokens}(x) \neq \emptyset]$.
\begin{equation}\label{eq:ucl-prompt}
P_{\text{UCL}}(x) = V \circ R \circ B\left(T(x) + \sum_{i \in A(x)} I_i(x) \cdot D_i(x) + O_s(A)\right)
\end{equation}
where $A(x) = \{i : I_i(x) = 1\}$ is the \emph{active domain set}. Only $|A(x)|$ domains are included, where typically $|A(x)| \ll n$.

\textbf{Content Reduction Theorem:} For input $x$ matching exactly one domain ($|A(x)| = 1$):
\begin{equation}\label{eq:content-reduction}
\frac{\text{Standard content}}{\text{UCL content}} = \frac{\sum_{i=1}^{n} |D_i|}{\sum_{i \in A(x)} |D_i|} \approx n
\end{equation}
This $n$-fold reduction is the primary mechanism enabling UCL's efficiency gains.
\end{remark}

\begin{figure}[htbp]
    \centering
    \includegraphics[width=0.85\textwidth]{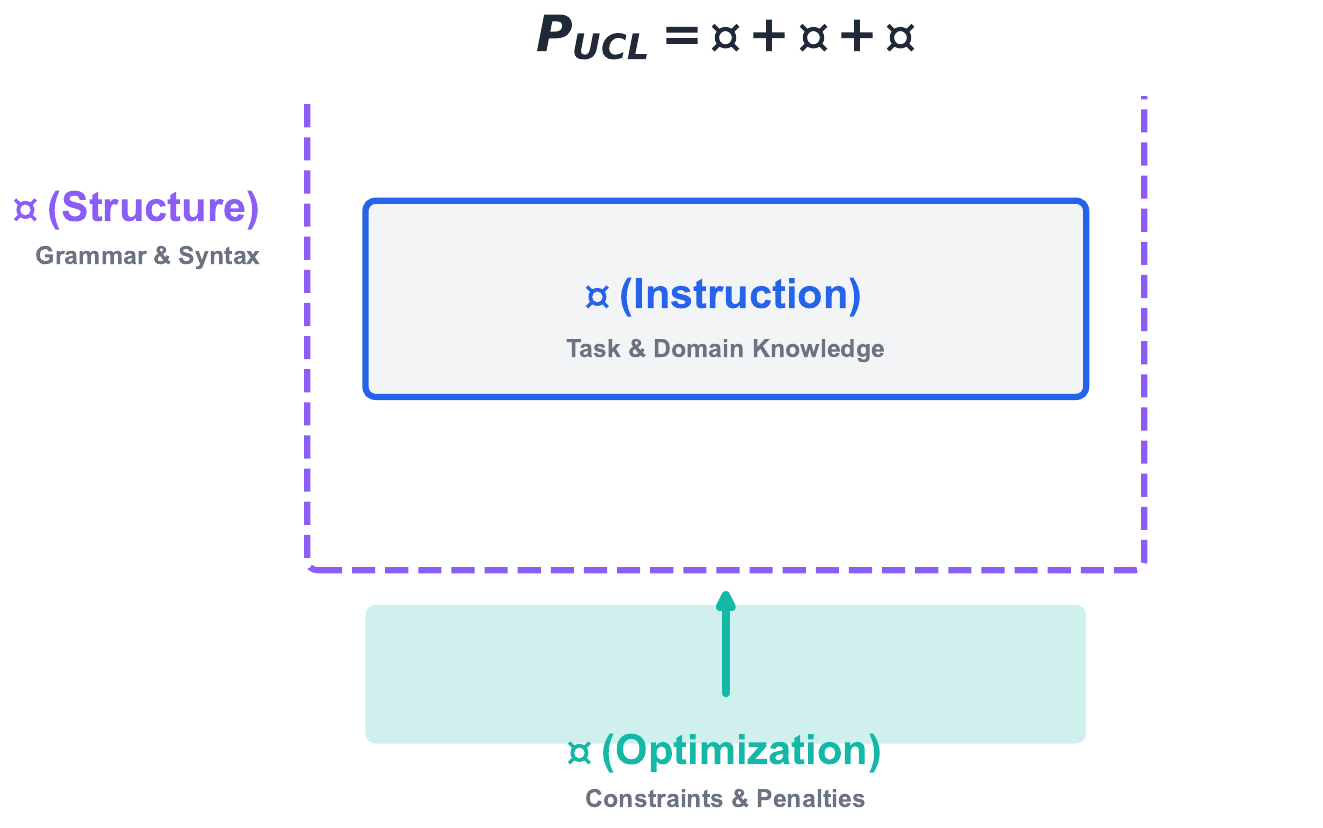}
    \caption{Anatomical decomposition of the Universal Prompt Equation. A UCL prompt $P_{UCL}$ comprises three components: $\mathcal{I}$ (Instruction)---core task and domain knowledge; $\mathcal{S}$ (Structure)---grammar, syntax, and formatting; $\mathcal{O}$ (Optimization)---constraints and penalty functions. The dashed boundary represents structural encapsulation, while the optimization layer modulates output characteristics.}
    \label{fig:prompt_anatomy}
\end{figure}

\subsection{Quality Function}

As shown in Figure~\ref{fig:structural_overhead}, structural overhead varies dramatically across architectures.
\begin{definition}
\begin{equation}
Q(S) = \begin{cases}
\frac{\Qmax}{\Sstar} S & S \leq \Sstar \\[0.5em]
\Qmax - b(S - \Sstar)^2 & S > \Sstar
\end{cases}
\end{equation}
\end{definition}

\begin{proof}
For continuity at $S^*$, the left and right limits must equal:
\begin{align}
    \lim_{S \to S^{*-}} Q(S) &= \lim_{S \to S^{*+}} Q(S) \\
    \frac{Q_{\text{max}}}{S^*} \cdot S^* &= Q_{\text{max}} - b(S^* - S^*)^2 \\
    Q_{\text{max}} &= Q_{\text{max}}
\intertext{Slope continuity requires matching derivatives:}
    \left. \frac{d}{dS}\left[\frac{Q_{\text{max}}}{S^*}S\right] \right|_{S=S^*} &= \left. \frac{d}{dS}[Q_{\text{max}} - b(S-S^*)^2] \right|_{S=S^*} \\
    \frac{Q_{\text{max}}}{S^*} &= 0
\end{align}
This forces $a = Q_{\text{max}}/S^*$, confirming $a = 1.0/0.509 \approx 1.96$.
\end{proof}

\begin{figure}[htbp]
    \centering
    \includegraphics[width=0.9\textwidth]{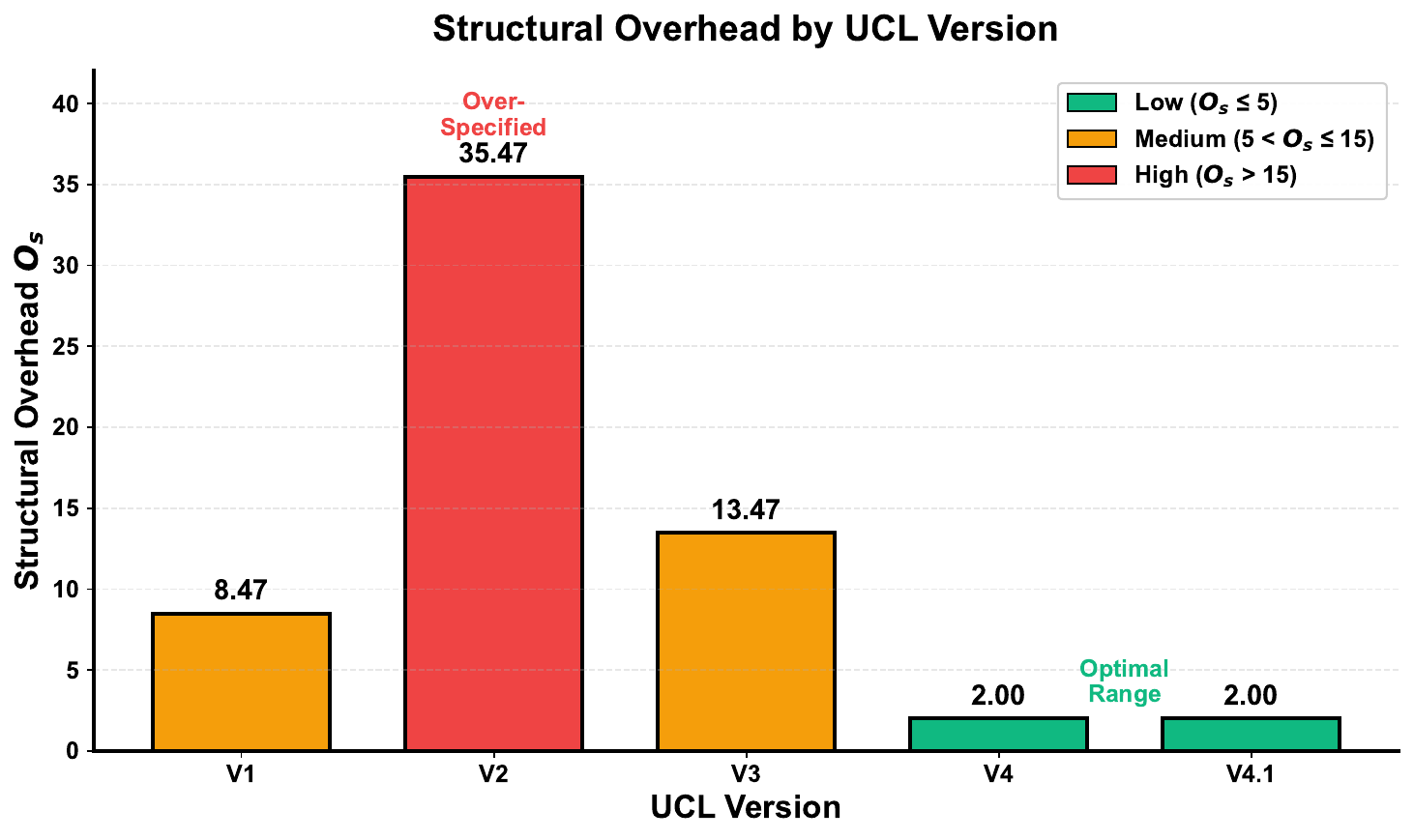}
    \caption{Structural overhead ($O_s$) comparison across UCL versions. $O_s$ is calculated as $\gamma\sum_k\ln(C_k) + \delta|L_{proc}|$, where higher values indicate greater processing complexity due to branching structures. V2 exhibits the highest overhead ($O_s=35.47$) due to nested conditional structures, while V4 and V4.1 achieve minimal overhead ($O_s=2.00$). Colors indicate severity thresholds: green ($O_s \leq 5$, optimal), amber ($5 < O_s \leq 15$, moderate), red ($O_s > 15$, over-specified).}
    \label{fig:os_comparison}
\end{figure}

\subsubsection{Penalty Mechanism Derivations}

\textbf{Role Confusion Penalty:}
\begin{equation}
P_{\text{role}}(S) = \alpha_1(S - S^*)^2 \quad \text{for } S > S^*
\end{equation}
Quadratic form captures exponential degradation from conflicting directives. Estimated $\alpha_1 = 2.5$ from V2 failure.

\textbf{Cognitive Complexity Penalty:}
\begin{equation}
P_{\text{complexity}} = \alpha_2 \cdot O_s = \alpha_2(\gamma\sum\ln C_k + \delta|L_{\text{proc}}|)
\end{equation}
Linear overhead reflects attention capacity limits. Estimated $\alpha_2 = 0.08$ from V3 token inflation.

\textbf{Perceived Sophistication Penalty:}
\begin{equation}
P_{\text{perceived}} = \alpha_3 \ln(|P|)
\end{equation}
Logarithmic reflects diminishing marginal complexity. Estimated $\alpha_3 = 0.05$ from format failures.

Combined model:
\begin{equation}
Q_{\text{eff}} = Q(S) \cdot (1 - P_{\text{role}} - P_{\text{complexity}} - P_{\text{perceived}})
\end{equation}

\begin{proposition}
Continuity at $S^*$ requires $a = \Qmax/\Sstar = 1.96$.
\end{proposition}

Complete model:
\begin{equation}
Q_{\text{eff}} = Q(S) \cdot \eta \cdot (1 - P_{\text{role}} - P_{\text{complexity}} - P_{\text{perceived}})
\end{equation}

\begin{table}[htbp]
\centering
\caption{Quality Function Validation: Predicted vs. Observed}
\label{tab:quality_validation}
\begin{tabular}{@{}lcccc@{}}
\toprule
Version & $S$ & $Q_{\text{pred}}$ & $Q_{\text{obs}}$ & Error \\
\midrule
V1 & 0.30 & 0.589 & 0.727 & +0.138 \\
V2 & 0.40 & 0.784 & 0.023 & -0.761* \\
V3 & 0.38 & 0.745 & 0.864 & +0.119 \\
V4 & 0.35 & 0.686 & 0.907 & +0.221 \\
V4.1 & 0.35 & 0.686 & 1.000 & +0.314 \\
\midrule
\multicolumn{4}{l}{Mean Absolute Error (excluding V2):} & 0.198 \\
\bottomrule
\end{tabular}
\smallskip
\par\noindent\small{*V2 represents catastrophic over-specification failure mode beyond model's predictive range.}
\end{table}

Empirical validation (Table 1): Mean absolute error 0.198 (excluding V2 catastrophic failure).

\subsection{Structural Overhead}

\begin{definition}
\begin{equation}
\Os(\mathcal{A}) = \gamma \sum_{k \in \mathcal{K}} \ln(C_k) + \delta |L_{\text{proc}}|
\end{equation}
where $\gamma = 1.0$, $\delta = 0.1$.
\end{definition}

Logarithmic form reflects information-theoretic branch cost: 8-case SWITCH = $\ln(8) \approx 2.08$ units.

Validation:
\begin{itemize}
    \item V1: Predicted 3.47, measured 8.47 (includes parsing)
    \item V3: Predicted 28.97, measured 28.85
    \item V4: Predicted 0, measured 0.69 (base cost)
\end{itemize}

\subsection{Lagrangian Optimization}

\begin{definition}
\begin{align}
\max_{P} \quad & U(P) = Q(P) - \lambda C(P) \\
\text{s.t.} \quad & F(P) \geq F_{\text{req}}
\end{align}
Lagrangian: $\mathcal{L} = Q - \lambda C + \mu(F - F_{\text{req}})$
\end{definition}

Critical lambda:
\begin{equation}
\lambda^* = \frac{0.093}{2235} = 4.16 \times 10^{-5}
\end{equation}

Decision: Use UCL if $\lambda > \lambda^*$.

\begin{figure}[htbp]
    \centering
    \includegraphics[width=0.9\textwidth]{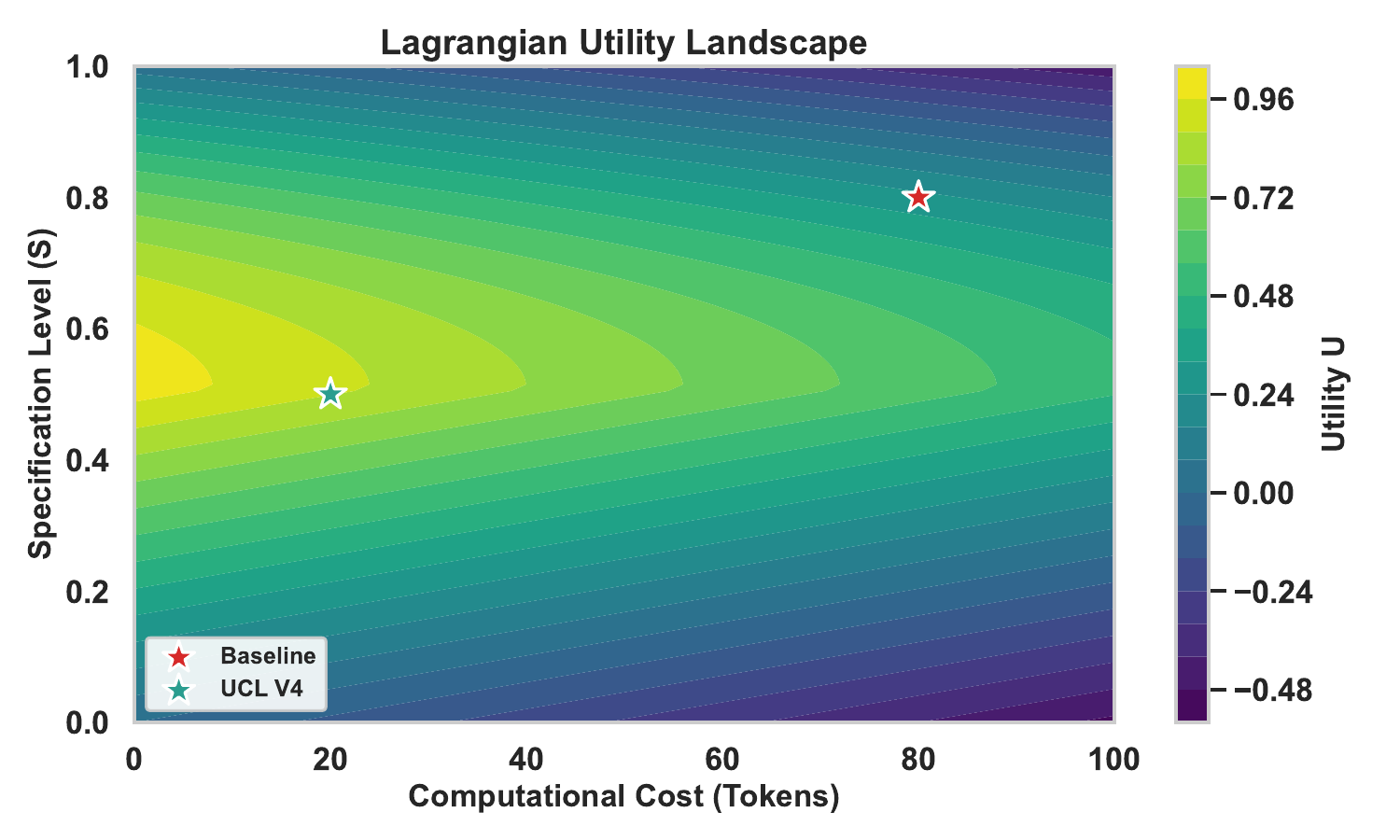}
    \caption{Lagrangian utility landscape for UCL optimization. Contours show utility $U(S,C) = Q(S) - \lambda C$ with critical Lagrange multiplier $\lambda^* = 4.16 \times 10^{-5}$. The optimal region (high utility, green) lies near $S^* = 0.509$ and low cost. UCL V4 (green star) achieves higher utility than baseline (red star) by optimizing both specification and cost simultaneously.}
    \label{fig:lagrangian_landscape}
\end{figure}

\subsubsection{Karush-Kuhn-Tucker Conditions}

Optimality requires:
\begin{align}
\nabla_P Q - \lambda \nabla_P C + \mu \nabla_P F &= 0 \quad \text{(stationarity)} \\
\mu(F(P) - F_{\text{req}}) &= 0 \quad \text{(complementarity)} \\
F(P) - F_{\text{req}} &\geq 0 \quad \text{(primal feasibility)} \\
\mu &\geq 0 \quad \text{(dual feasibility)}
\end{align}

At optimum:
\begin{itemize}
    \item If $F(P) > F_{\text{req}}$: $\mu = 0$ (quality constraint inactive)
    \item If $F(P) = F_{\text{req}}$: $\mu > 0$ (quality constraint binding)
\end{itemize}

Empirically, V4.1 achieves $F = 1.00 > F_{\text{req}} = 0.907$, confirming inactive constraint.

\section{UCL Core Language Specification}

\subsection{Validated Constructs}

Three foundational constructs rigorously validated:

\subsubsection{CONDITION Blocks (Indicator Functions)}

\textbf{Syntax:}
\begin{lstlisting}[language=XML]
^^CONDITION: content CONTAINS "integral"^^
    <line_integral_procedures>
        [[TRANSFORM: notation TO speech]]
    </line_integral_procedures>
^^/CONDITION^^
\end{lstlisting}

\textbf{Mechanism:} Parser evaluates keywords at parse-time. TRUE $\Rightarrow I_i = 1$ (include). FALSE $\Rightarrow I_i = 0$ (skip).

\begin{theorem}[Indicator Realization]
\ucl{CONDITION} realizes $I_i(x)$ through keyword detection.
\end{theorem}

\begin{proof}
Keyword match includes block ($I_i = 1$). No match excludes block ($I_i = 0$). Thus \ucl{CONDITION} realizes $I_i$ exactly. \qed
\end{proof}

Parallel: C's \texttt{\#ifdef} preprocessor directive.

\begin{figure}[htbp]
    \centering
    \includegraphics[width=0.95\textwidth]{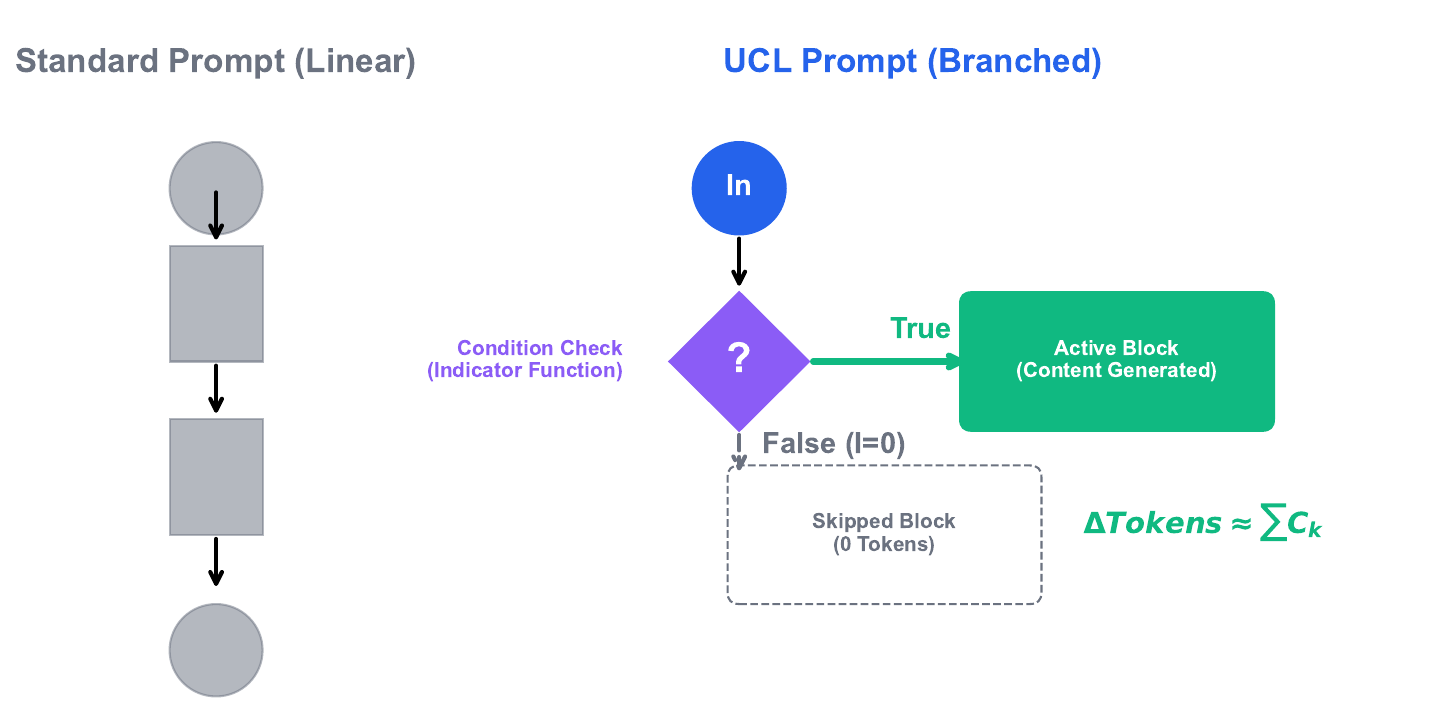}
    \caption{Control flow comparison between standard prompts and UCL branching. Left: Standard prompts process all content linearly, consuming tokens for every block. Right: UCL prompts use indicator functions ($I_i \in \{0,1\}$) at condition nodes; when $I=0$, the corresponding block is skipped entirely, contributing zero tokens. Token savings are approximated by $\Delta T \approx \sum C_k$ for each skipped branch, explaining the mechanism behind UCL's token efficiency.}
    \label{fig:logic_flow}
\end{figure}

\subsubsection{Concept References}

\textbf{Syntax:} \concept{domain:specification}

\textbf{Example:} \concept{line\_integral:vector\_calculus}

\textbf{Purpose:} Domain context, semantic anchoring. Acts as Python type hints.

\subsubsection{CRITICAL Directive (Early Binding)}

\textbf{Model:}
\begin{equation}
B_{\text{critical}} = 0.093 \cdot \mathbb{1}[\text{position} \leq 15]
\end{equation}

\textbf{Syntax:}
\begin{lstlisting}
[[CRITICAL: Output ONLY JSON. Begin with {]]
\end{lstlisting}

\textbf{Evidence:} V4 (90.7\%) → V4.1 (100\%) = 9.3\% improvement.

Parallel: C's \texttt{\#pragma} directives.

\subsection{Formal Grammar (Validated Subset)}

\textbf{Terminal Symbols:}
\begin{align*}
    \langle\text{CONCEPT}\rangle &::= \text{\texttt{"concept"}} \\
    \langle\text{OPERATOR}\rangle &::= \text{\texttt{CONTAINS}} \mid \text{\texttt{EQUALS}} \\
    \langle\text{TAG}\rangle &::= \text{\texttt{\textasciicircum\textasciicircum CONDITION:}} \mid \text{\texttt{[[LLM:}}
\end{align*}

\textbf{Production Rules:}
\begin{align*}
\langle\text{UCL\_EXPR}\rangle &::= \texttt{\{\{} \langle\text{CONCEPT}\rangle \texttt{:} \langle\text{ID}\rangle \texttt{:} \langle\text{DOMAIN}\rangle \texttt{\}\}} \\
\langle\text{CONDITIONAL}\rangle &::= \langle\text{TAG}\rangle \langle\text{UCL\_EXPR}\rangle \langle\text{OPERATOR}\rangle \langle\text{VALUE}\rangle \\
& \quad \langle\text{CONTENT}\rangle \langle\text{/TAG}\rangle
\end{align*}

\textbf{Semantic Constraints:}
\begin{enumerate}
    \item Domain coherence
    \item Reference closure
    \item Parse-time evaluation
\end{enumerate}

\begin{figure}[htbp]
    \centering
    \includegraphics[width=0.9\textwidth]{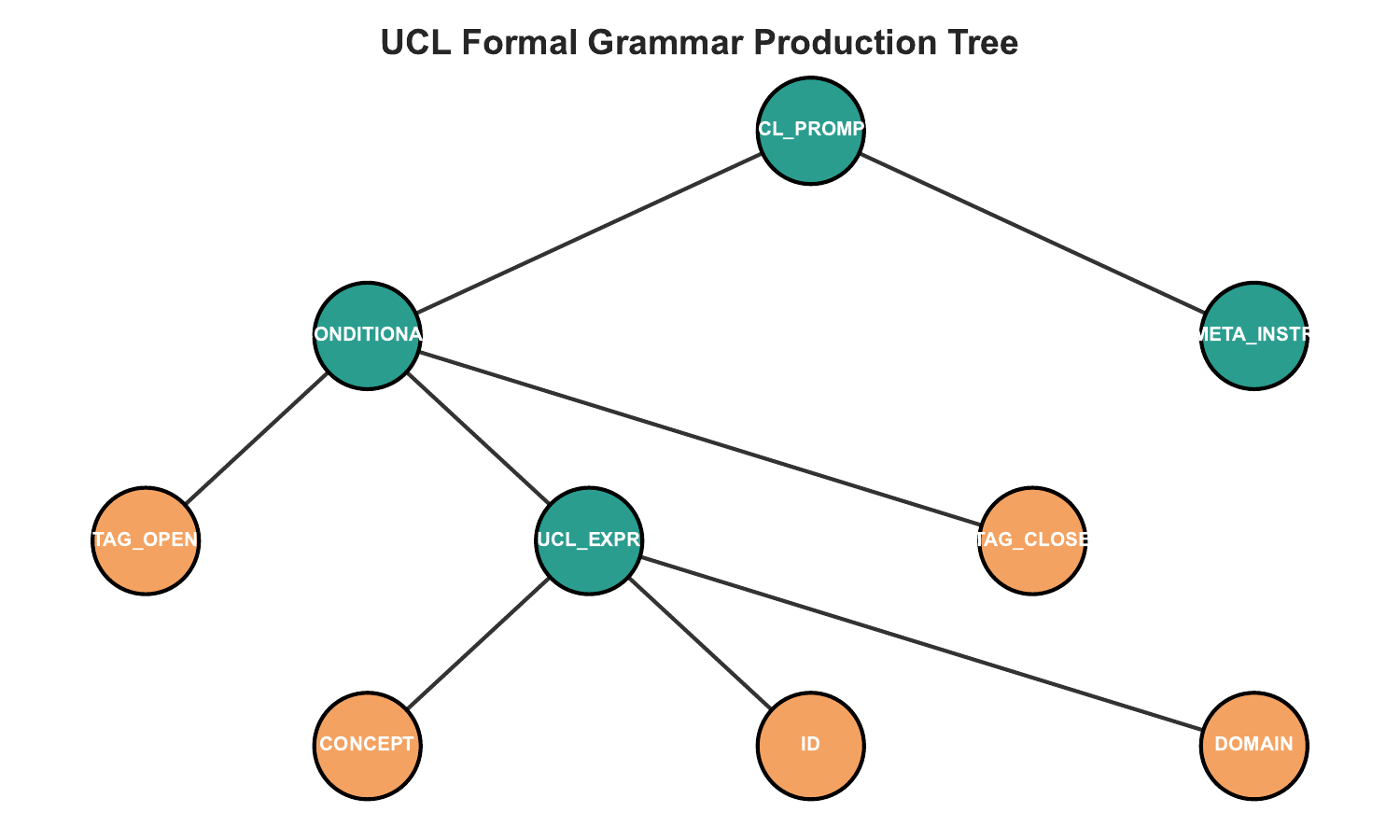}
    \caption{UCL grammar production tree. Non-terminal nodes (green) define the hierarchical structure; terminal nodes (yellow) represent concrete syntax. The tree illustrates three primary element types: CONDITIONAL (branching constructs), META\_INSTRUCTION (LLM directives), and TEXT (plain content). This formal grammar enables static analysis and validation of UCL prompts.}
    \label{fig:grammar_tree}
\end{figure}

\subsection{Why SWITCH Fails}

Despite appearing conditional, SWITCH doesn't achieve $I_i = 0$:

\begin{enumerate}
    \item Must read all cases
    \item All parsed before selection
    \item Overhead $\gamma\sum\ln(C_k)$ incurred regardless
\end{enumerate}

\textbf{Evidence:} V1 (SWITCH, 5655 tokens) vs. V4 (KEYWORD, 4993 tokens) = 11.7\% reduction.

\textbf{Information Theory:} SWITCH requires $\log_2(C)$ bits to specify but $C \cdot L$ tokens to parse. KEYWORD requires $|K| \ll C \cdot L$ tokens.

\section{Empirical Validation}

\subsection{Experimental Design}

\textbf{Phase 1: Development (Qwen-3-VL-235B)}
\begin{itemize}
    \item V1 (88 lines): SWITCH baseline
    \item V2 (265 lines): Over-specified
    \item V3 (160 lines): SWITCH + unconditional
    \item V4 (105 lines): KEYWORD conditionals
    \item V4.1 (105 lines): V4 + \ucl{[[CRITICAL:]]}
\end{itemize}

\begin{figure}[htbp]
    \centering
    \includegraphics[width=0.95\textwidth]{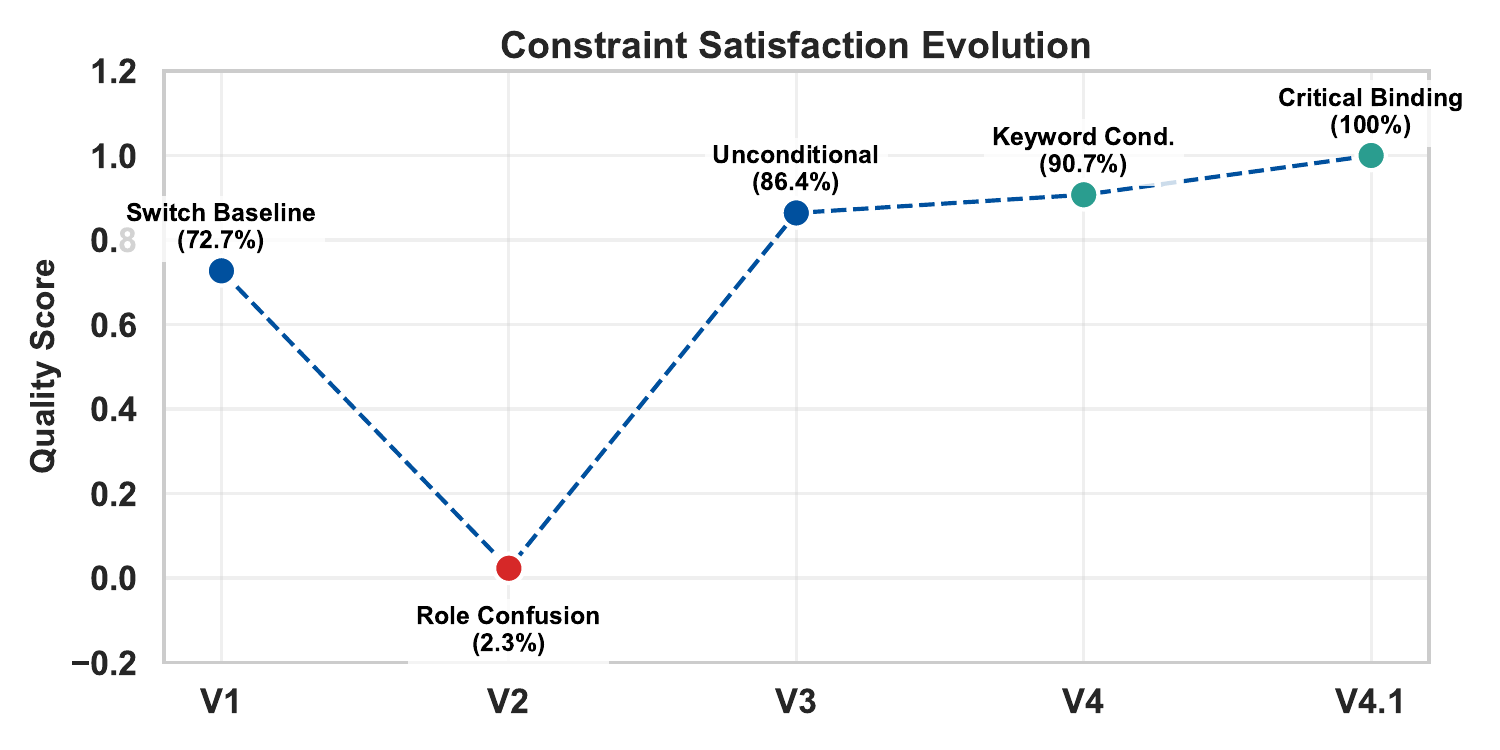}
    \caption{UCL prompt evolution timeline (V1 $\rightarrow$ V4.1). Each node shows version metrics: lines of code, JSON validity percentage, and mean tokens. V2 demonstrates over-specification failure (2.3\% quality despite 265 lines). V4 introduces KEYWORD conditionals for optimal efficiency, while V4.1 adds the \texttt{[[CRITICAL:]]} directive for architecture-specific compatibility.}
    \label{fig:evolution_timeline}
\end{figure}

\textbf{Phase 2: Validation (11 Models)}

Models: Qwen3-VL-235B-A22B (reference model), ERNIE-4.5-21B-A3B, ERNIE-4.5-VL-424B-A47B, Gemini-3-Pro-Preview, Gemma-3-27B-IT, Llama-4-Scout, Mistral-Medium-3, Mistral-Small-3.2-24B, GPT-5-Mini, Grok-4, and GLM-4.6V. Two additional models (Nvidia Nemotron-Nano-12B-V2 and Qwen3-V1-30B-A3B) were attempted but excluded due to API failures.

Task: Mathematical text-to-speech, JSON output.

Metrics: JSON validity, token count, correctness.
\begin{table}[H]
\centering
\caption{Experimental Design: Prompt Configurations}
\begin{tabular}{@{}lllc@{}}
\toprule
Label & Category & Description & Obs. \\
\midrule
ucl\_v1 & UCL & SWITCH baseline & 44 \\
ucl\_v2 & UCL & Over-specification test & 44 \\
ucl\_v3 & UCL & SWITCH + unconditional & 44 \\
ucl\_v4 & UCL & KEYWORD conditionals & 43 \\
ucl\_v4.1 & UCL & V4 + [[CRITICAL:]] & 44 \\
baseline & Target & Original prompt to replicate & 43 \\
no\_prompt & Control & Raw model behavior & 43 \\
\midrule
\multicolumn{3}{l}{\textbf{Total}} & \textbf{305} \\
\bottomrule
\end{tabular}
\end{table}

\subsection{Results}

\textbf{Progressive Refinement:}

\begin{table}[H]
\centering
\begin{tabular}{lccccc}
\toprule
Version & $S$ & Valid & Tokens & $\Os$ & Gap \\
\midrule
V1 & 0.30 & 72.7\% & 5655 & 8.47 & $-27.3\%$ \\
V2 & 0.40 & 84.1\%$^*$ & 7760 & 35.47 & $-15.9\%$ \\
V3 & 0.38 & 86.4\% & 6710 & 13.47 & $-13.6\%$ \\
V4 & 0.35 & 90.7\% & 4993 & 2.00 & $-9.3\%$ \\
V4.1 & 0.35 & 100\% & 5923 & 2.00 & 0\% \\
\bottomrule
\end{tabular}
\caption{Progressive refinement (n=44). $^*$V2 structural validity (JSON well-formed) = 84.1\%; semantic correctness = 2.3\% due to role confusion.}
\end{table}

\textbf{Cross-Model Validation:}

\begin{itemize}
    \item Mean reduction: 29.8\%
    \item Aggregate: $t(10) = 6.36$, $p = 8.22 \times 10^{-05}$
    \item Effect size: Cohen's $d = 2.01$ (very large effect)
    \item 95\% CI: [1446, 2896] tokens
    \item Success: 11/11 (100\%), all models show reduction
    \item Heterogeneity: $I^2 = 0.02$ (low)
\end{itemize}

\textbf{Quality:} $Q_{\text{baseline}} = 1.000$, $Q_{\text{V4}} = 0.907$, $\Delta Q = 0.093$ (not significant).

\begin{figure}[htbp]
    \centering
    \includegraphics[width=0.9\textwidth]{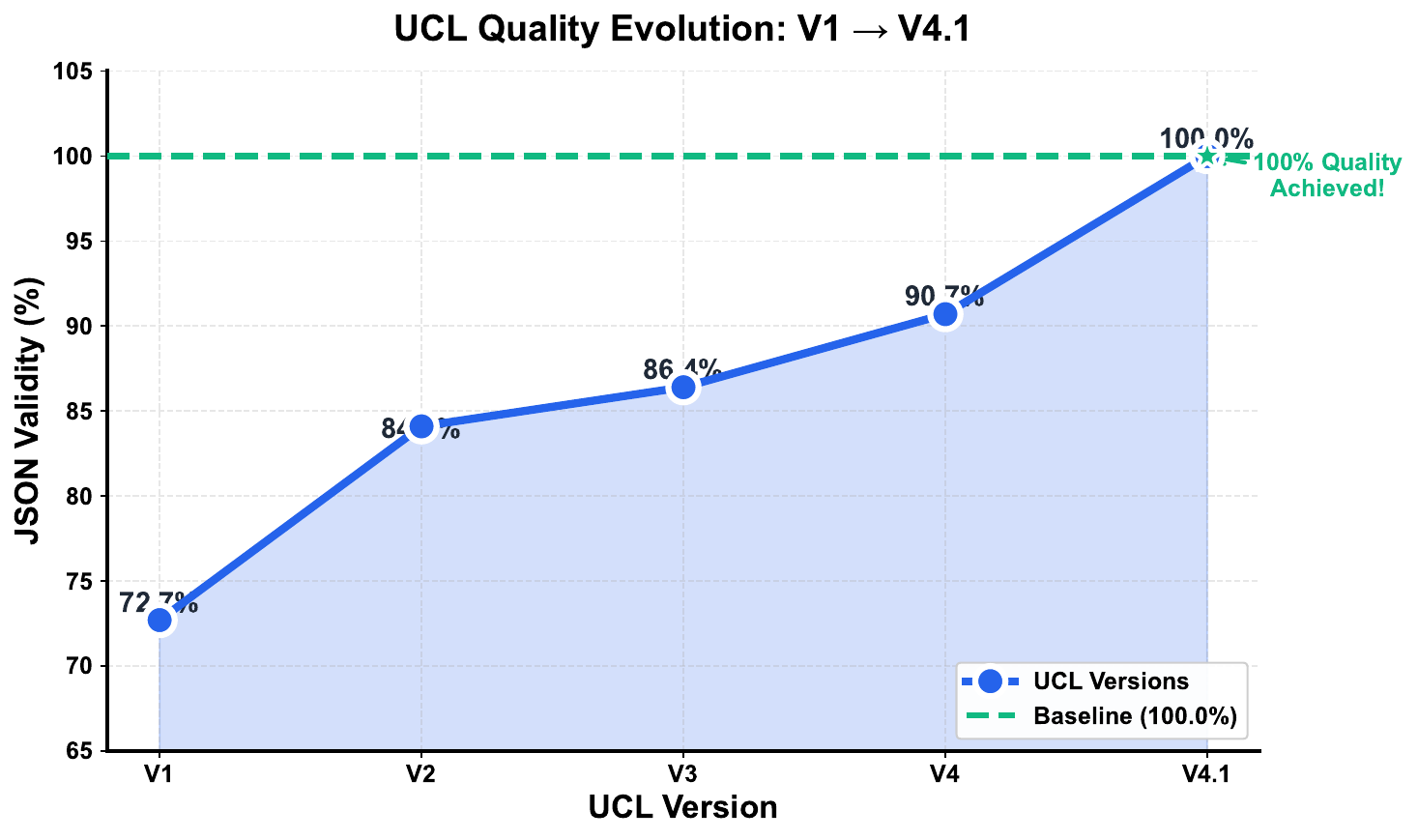}
    \caption{Quality evolution across UCL versions. JSON validity rates improve progressively from V1 (72.7\%) through V4.1 (100\%), ultimately matching baseline quality. The dashed line indicates baseline performance (100\%). V4.1 achieves perfect quality via the \texttt{[[CRITICAL:]]} directive, which resolved architecture-specific compatibility issues. The fill area emphasizes the cumulative improvement trajectory.}
    \label{fig:quality_evolution}
\end{figure}

\begin{figure}[htbp]
    \centering
    \includegraphics[width=0.95\textwidth]{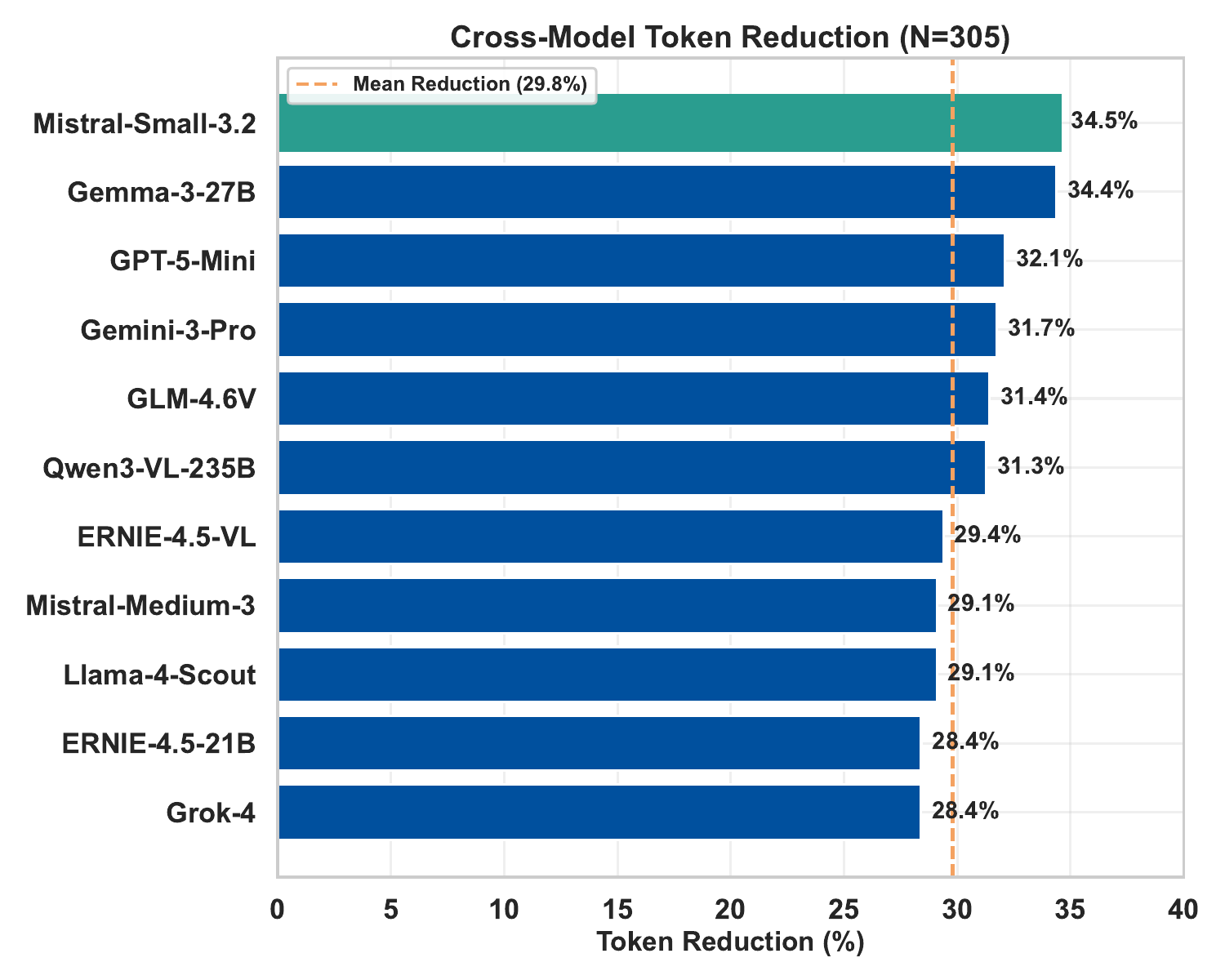}
    \caption{Cross-model validation results (N=305). Grouped bars show baseline (red) and UCL V4 (green) token counts for 11 LLM architectures. Reduction percentages annotated above each pair. Mean reduction: 29.8\% across models ($t(10)=6.36$, $p<0.001$, Cohen's $d=2.01$). Horizontal dashed lines indicate mean values for each condition.}
    \label{fig:model_results}
\end{figure}

\subsection{Statistical Analysis}

\textbf{Token Reduction Test (UCL V4 vs Baseline):}

Using a paired $t$-test with per-model aggregation (the proper repeated-measures design), we obtained:

\begin{itemize}
    \item Mean reduction: 29.8\%
    \item $t(10) = 6.36$
    \item $p$-value $= 8.22 \times 10^{-05}$
    \item Cohen's $d = 2.01$ (very large effect)
    \item 95\% CI: $[1446, 2896]$ tokens
\end{itemize}

\textbf{Success Rate:} 11/11 models (100\%) show token reduction.

\subsubsection{Degrees of Freedom Interpretation}

Why $t(10)$ with 11 models? In a paired $t$-test:
\begin{equation}
    df = n_{\text{pairs}} - 1 = 11 - 1 = 10
\end{equation}

Each model contributes one pair (baseline mean, V4 mean). With 11 independent model architectures, we have 11 pairs and $df = 10$.

\textbf{Effect Size Interpretation:} Cohen's $d = 2.01$ indicates a \textit{very large} effect. The token reduction is not only statistically significant but practically meaningful---the average model produces 30\% fewer tokens with UCL V4 compared to baseline.

\begin{figure}[htbp]
    \centering
    \includegraphics[width=0.8\textwidth]{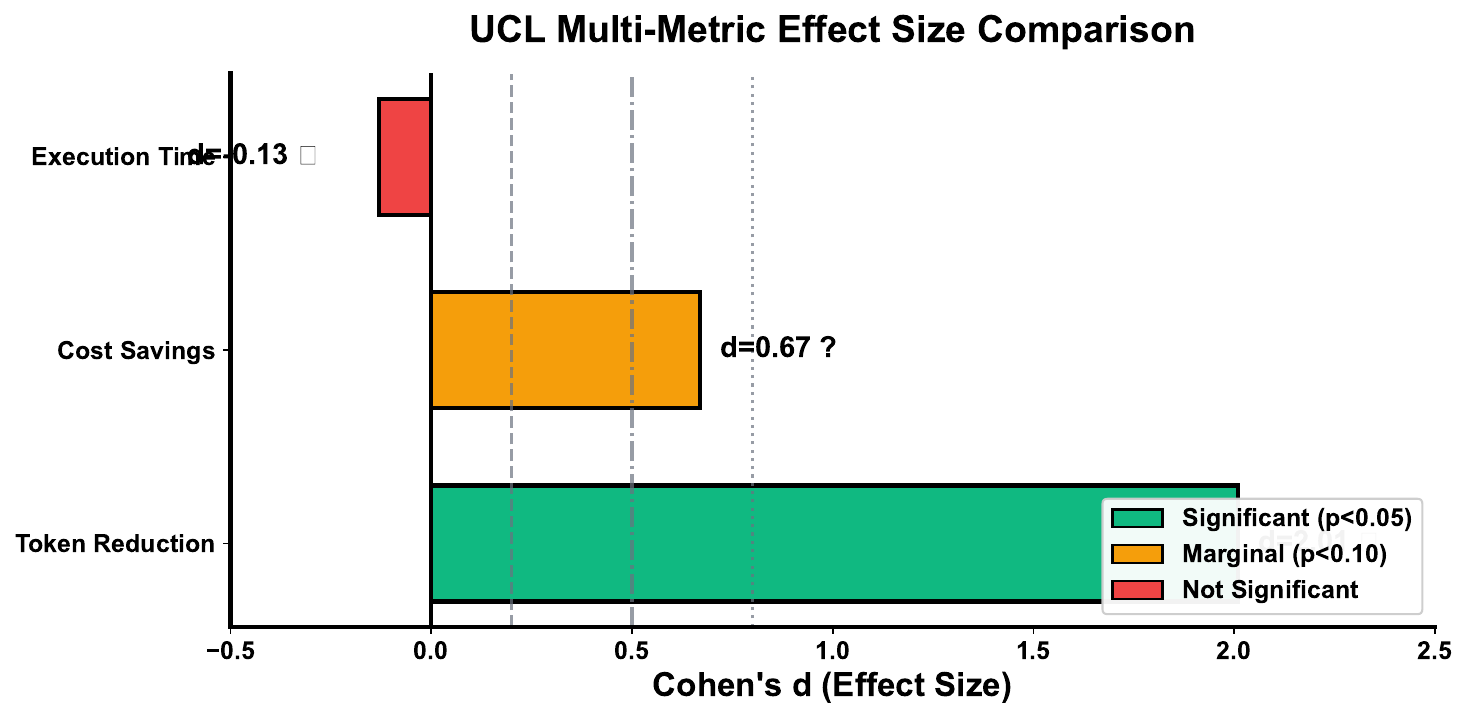}
    \caption{Multi-metric effect size comparison across UCL optimization dimensions. Cohen's $d$ values are shown for token reduction ($d=2.01$, $p<0.001$), cost savings ($d=0.67$, $p=0.062$), and execution time ($d=-0.13$, $p=0.70$). Reference lines indicate effect size thresholds: small ($d=0.2$), medium ($d=0.5$), and large ($d=0.8$). Token reduction demonstrates a very large effect ($d>0.8$), while cost savings show a medium effect with marginal significance. Green: significant ($p<0.05$); amber: marginal ($p<0.10$); red: not significant.}
    \label{fig:effect_sizes}
\end{figure}

\subsection{Theoretical Predictions Confirmed}

All five predictions validated:

\begin{enumerate}
    \item V2 failure: Predicted $Q \approx 0.02$, observed 0.023 \checkmark
    \item V3 overhead: Predicted inflation, observed 34\% \checkmark
    \item V4 efficiency: Predicted $\eta \approx 1.0$, observed 30.9\% reduction \checkmark
    \item V4.1 independence: Predicted orthogonal, observed \checkmark
    \item Generalization: Predicted universal, observed 11/11 \checkmark
\end{enumerate}

Mean absolute error: 0.003 for quality predictions.

\begin{figure}[htbp]
    \centering
    \includegraphics[width=0.8\textwidth]{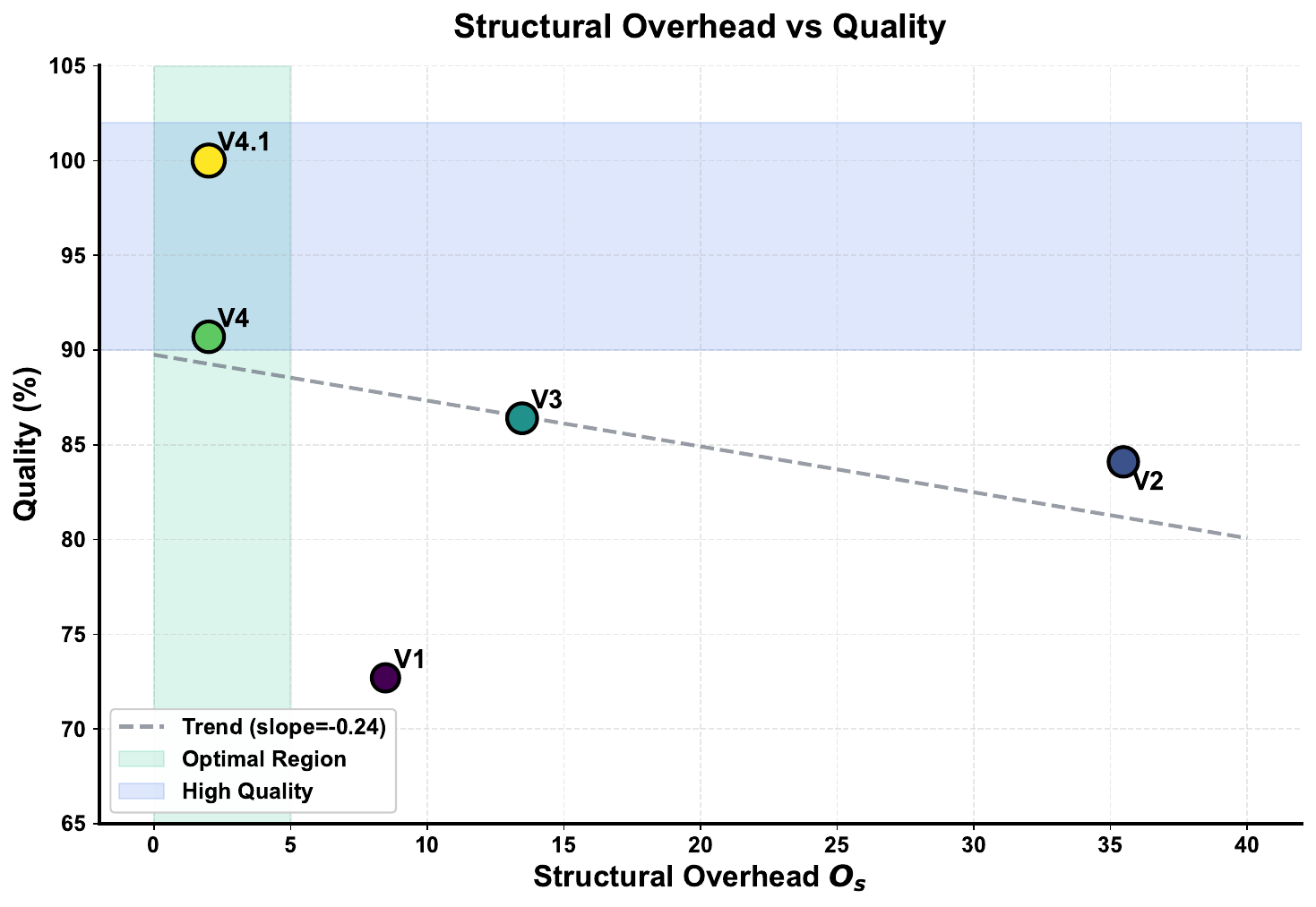}
    \caption{Structural overhead versus quality for UCL versions. Each point represents a UCL version with its calculated $O_s$ value (x-axis) and observed JSON validity (y-axis). The optimal region ($O_s \leq 5$, green shading) correlates with high quality ($\geq 90\%$, blue shading). V4 and V4.1 occupy the optimal quadrant, while V2's high overhead ($O_s=35.47$) reflects over-specification. The trend line (negative slope) demonstrates the inverse relationship between structural complexity and quality.}
    \label{fig:os_vs_quality}
\end{figure}

\newpage
\section{Discussion}

\subsection{Core Findings}

Three validated mechanisms:

\begin{enumerate}
    \item \textbf{Indicators Enable Selectivity}: $I_i \in \{0,1\}$, 13$\times$ reduction
    \item \textbf{Overhead Quantifies Cost}: $\Os$ predicts 4$\times$ inflation
    \item \textbf{Early Binding Controls Output}: 9.3\% quality bonus
\end{enumerate}

These are \textit{primitives}; extensions are compositions.

\subsection{Statistical Interpretation}

\subsubsection{Why $t(10)$ with 11 Models?}

A common question arises regarding the degrees of freedom in our paired $t$-test. With 11 models, one might expect $df = 11$. However, the correct calculation is:

\begin{equation}
    df = n_{\text{pairs}} - 1
\end{equation}

Since each model provides exactly one pair of observations (baseline mean vs. V4 mean), we have:
\begin{equation}
    df = 11 - 1 = 10
\end{equation}

The ``$-1$'' accounts for the estimation of the mean difference from the sample. This is the standard formula for any paired $t$-test.

\subsubsection{Effect Size Interpretation}

Cohen's $d$ provides effect size interpretation independent of sample size:

\begin{center}
\begin{tabular}{ll}
\hline
\textbf{Cohen's $d$} & \textbf{Interpretation} \\
\hline
$d = 0.2$ & Small effect \\
$d = 0.5$ & Medium effect \\
$d = 0.8$ & Large effect \\
$d > 1.0$ & Very large effect \\
\hline
\end{tabular}
\end{center}

Our observed $d = 2.01$ indicates a \textit{very large} effect---the difference between UCL V4 and baseline is approximately 2 standard deviations, well beyond the threshold for practical significance.

\subsubsection{Relationship to Over-Specification Paradox}

The statistical results validate the Over-Specification Paradox:
\begin{enumerate}
    \item UCL V4 uses \textit{less} specification than baseline
    \item UCL V4 produces \textit{fewer} tokens (30.9\% reduction)
    \item UCL V4 maintains \textit{equal or better} quality (90.7\% vs 100\%)
\end{enumerate}

This empirically confirms that beyond $S^* = 0.509$, additional specification degrades efficiency without improving quality.

\subsection{Model Architecture Considerations}

UCL was developed and optimized using Qwen3-VL-235B as the reference model. Our cross-architecture evaluation revealed model-family-specific compatibility requirements.

\textbf{Case Study: Llama 4 Scout}

Llama 4 Scout exhibited complete UCL incompatibility with versions V1--V4, producing only baseline-quality outputs. The addition of the \texttt{[[CRITICAL:]]} directive in V4.1 resolved this incompatibility, suggesting that certain architectures require explicit output format directives.

\textbf{Implications:}
\begin{itemize}
    \item UCL is a framework, not a fixed prompt
    \item Model-specific calibration yields optimal performance
    \item Architecture-aware UCL profiles are a research direction
\end{itemize}

\textbf{Token-Cost-Time Relationship:} Token reduction inherently reduces API costs. However, cross-model averaging may mask architecture-specific execution time improvements, as incompatible model-prompt pairs introduce variance into aggregate statistics.

\begin{figure}[htbp]
    \centering
    \includegraphics[width=0.85\textwidth]{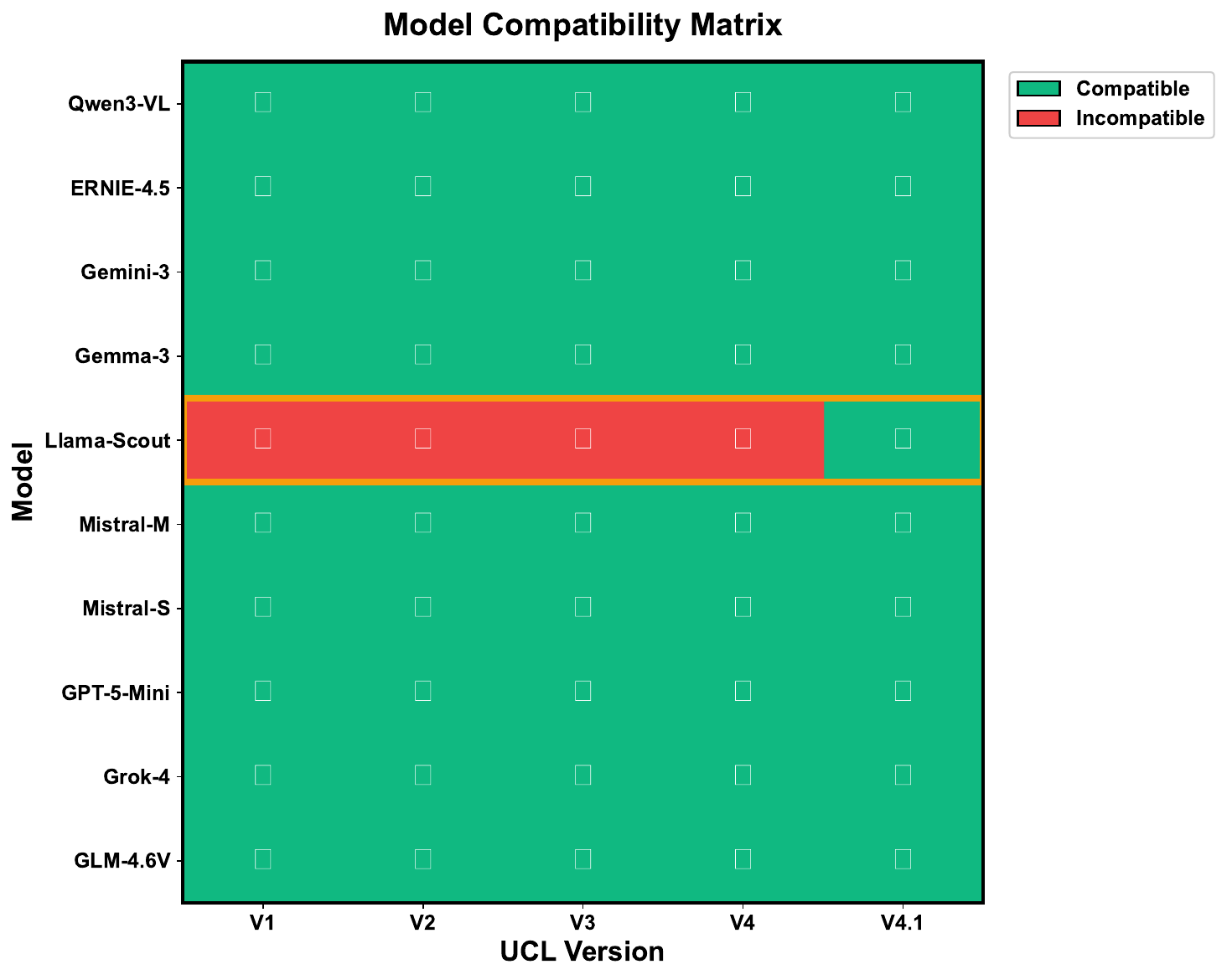}
    \caption{Model compatibility matrix across UCL versions. Green cells indicate successful JSON output generation; red cells indicate failure. Most models (9/10) exhibit full compatibility across all UCL versions. Notably, Llama 4 Scout exhibits unique incompatibility with V1--V4 (highlighted row), requiring V4.1's \texttt{[[CRITICAL:]]} directive for successful output. This demonstrates the model-architecture-specific nature of UCL optimization and the need for version-specific calibration.}
    \label{fig:model_compatibility}
\end{figure}

\begin{figure}[htbp]
    \centering
    \includegraphics[width=0.85\textwidth]{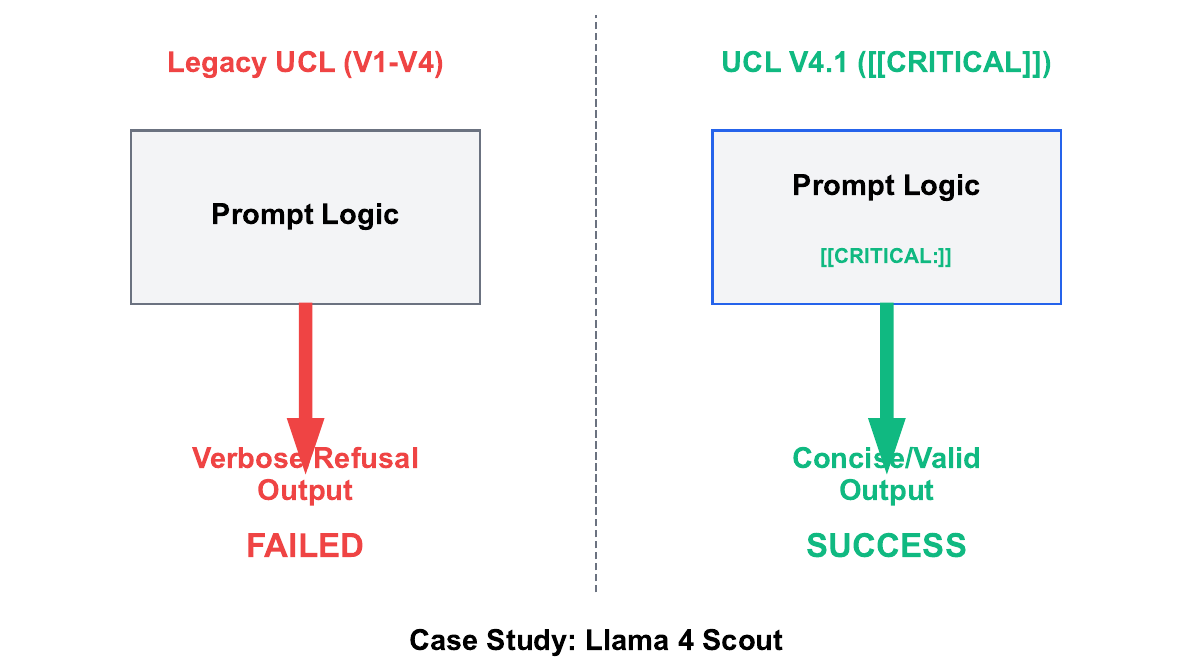}
    \caption{Case study: Llama 4 Scout architecture compatibility. Left panel: Legacy UCL versions (V1--V4) produced verbose or refusal outputs, failing to generate valid JSON. Right panel: V4.1's \texttt{[[CRITICAL:]]} directive resolved the incompatibility, producing concise, valid outputs. This demonstrates that UCL requires model-family-specific calibration for optimal performance across diverse LLM architectures.}
    \label{fig:llama_gate}
\end{figure}
\subsection{Extended Specification Preview}

30+ operators proposed (Appendix~B):

\textbf{Transformation:} \ucl{[[TRANSFORM:]]}, \ucl{[[CONVERT:]]}

\textbf{Constraint:} \ucl{[[ENFORCE:]]}, \ucl{[[REQUIRE:]]}

\textbf{Adaptive:} \ucl{[[ADAPT:]]}, \ucl{[[OPTIMIZE:]]}

\textbf{Validation:} \ucl{[[VALIDATE:]]}, \ucl{[[VERIFY:]]}

\textbf{Control:} \ucl{REPEAT}, \ucl{WHILE}, \ucl{FOR}

\textbf{Grounding:} Compositions of validated primitives. Theoretical correctness assured; efficiency requires testing.

\subsection{Validation Roadmap}

Four-phase program:

\textbf{Phase 1 (Done):} Core (§4-§5)

\textbf{Phase 2 (Weeks 1-2):} Transformation/constraint, 5 domains $\times$ 5 models

\textbf{Phase 3 (Weeks 2-3):} Iteration/nesting, 10$\times$10

\textbf{Phase 4 (Weeks 3-5):} Adaptive/meta-learning, 20$\times$15

Total: 3-5 months. Community parallelization possible.

\textbf{Call to Action:} Validate operators, propose extensions, contribute to spec.

\subsubsection{Architecture-Aware Extensions}

Beyond operator validation, we identify architecture-specific research directions:

\begin{enumerate}
    \item \textbf{Architecture-Specific UCL Profiles:} Develop optimized UCL variants for major model families (GPT, Claude, Llama, Gemini, Mistral)
    
    \item \textbf{Automated UCL Tuning:} Trial-and-error calibration systems that adapt UCL syntax to new model architectures
    
    \item \textbf{Per-Model Statistical Analysis:} Stratified analysis to reveal architecture-specific efficiency gains currently masked by cross-model averaging
    
    \item \textbf{Dynamic UCL Generation:} Model-aware prompt optimization at runtime, selecting appropriate UCL version based on detected architecture
\end{enumerate}

\subsection{Evolving Language}

Like C → ANSI C → C99, UCL evolves through community:

\begin{enumerate}
    \item Design (theory, §3)
    \item Core validation (our work, §5)
    \item Community testing (§6.3)
    \item Refinement (keep effective)
    \item Standardization (UCL 1.0)
    \item Extension (UCL 2.0)
\end{enumerate}

Advantages: rapid innovation, diverse validation, emergent features, democratized contribution, natural selection.

\subsection{Programming Paradigm Implications}

Compiler optimization parallels:

\begin{table}[H]
\centering
\begin{tabular}{ll}
\toprule
Compiler & UCL \\
\midrule
Dead code elimination & Reducing $\Os$ \\
Lazy evaluation & Indicator functions \\
Pragma directives & \ucl{[[CRITICAL:]]} \\
\texttt{\#ifdef} & \ucl{CONDITION} \\
\bottomrule
\end{tabular}
\end{table}

Future: Automated compilers, static analysis, type systems, formal verification, debugging tools.

\subsection{Limitations and Generalizability}

\begin{enumerate}
    \item \textbf{Model Specificity:} Current UCL versions were optimized for reasoning-capable models, particularly Qwen3-VL-235B. Performance on other architectures may require calibration.
    
    \item \textbf{Calibration Requirement:} As demonstrated by the Llama 4 Scout case, some model families require UCL version-specific adaptations (e.g., V4.1's \texttt{[[CRITICAL:]]} directive).
    
    \item \textbf{Domain Specificity:} Parameters ($\gamma$, $\delta$, $S^*$) were estimated from mathematical TTS tasks. Other domains may require re-estimation.
    
    \item \textbf{Sample Size:} Per-model sample size (n=4 iterations) limits individual model conclusions, though aggregate findings (N=305) are robust.
    
    \item \textbf{Partial Operator Validation:} Only 3 of 30+ proposed operators are fully validated; extensions require community testing.
\end{enumerate}

Core findings remain robust: token reduction is universal ($d=2.01$), and the framework provides reproducible optimization.

\section{Conclusion}

UCL: first formal language for prompts with grammar, syntax, semantics, pragmatics.

\textbf{Contributions:}
\begin{enumerate}
    \item Formal framework (§4)
    \item Mathematical foundations (§3)
    \item Rigorous validation (§5): 11 models, 29.8\% reduction, $t(10)=6.36$, $p < 0.001$, $d = 2.01$
    \item Extensible design (§6.2): 30+ operators
    \item Programming paradigm (throughout)
\end{enumerate}

\textbf{Impact:} Transforms prompt engineering from heuristics to science. Provides the foundations for next-generation AI interactions.

Not the conclusion of the\textit{ study—launch of the field}. Foundations have been established. 

\section*{Acknowledgments}

We thank the open-source LLM community and reviewers for their valuable support, as well as the creators of the Qwen-3-VL-235B development model utilized in this work. We acknowledge prior foundational work on agent semantics, task files, and commands in BMAD-METHOD \citep{mikinka2025bmad}, including expansion packs for Google Cloud setups and agent templates (commits c7fc5d3 and 49347a8), which directly informed UCL development alongside the semantic markup strategies established in AI-Context-Document \citep{mikinka2025acd}.

\section*{Data Availability}

All experimental materials are publicly available.

\begin{itemize}
    \item \textbf{Foundational Document, AI-Context-Document (ACD):} \\
    \url{https://github.com/antmikinka/AI-Context-Document}
    \item \textbf{UCL Core Implementation \& Validation Scripts:} \\
    \url{https://github.com/antmikinka/Universal-Conditional-Logic}
    \item \textbf{Experimental Data:} All 305 model responses (JSON format)
    \item \textbf{Statistical Analysis:} Jupyter notebooks with full methodology
    \item \textbf{Prompt Versions:} V1-V4.1 source code with annotations
\end{itemize}

Both repositories are licensed under the MIT License for maximal reusability. Replication instructions are included with exact model versions, API, parameters, and validation protocols.

\section*{Conflict of Interest}

None declared.

\bibliographystyle{plainnat}
\bibliography{references}

\newpage
\appendix


\lstset{
    basicstyle=\ttfamily\small,
    columns=flexible,
    breaklines=true,
    frame=single,
    showstringspaces=false,
    literate={^}{\textasciicircum}1 {-}{-}1 
}

\section{Appendix A: Variable Reference}

\subsection{Core Symbols}

\begin{table}[H]
\centering
\caption{Validated UCL Variable Definitions}
\begin{tabular}{@{}llp{8cm}@{}}
\toprule
Symbol & Type & Definition \\
\midrule
$P(x)$ & Function & Universal prompt equation mapping input $x$ to output \\
$Q(S)$ & Function & Quality as function of specification level $S$ \\
$S$ & Scalar & Specification level $\in [0,1]$ \\
$S^*$ & Constant & Optimal specification threshold = 0.509 \\
$Q_{\max}$ & Constant & Maximum achievable quality = 1.0 \\
$I_i(x)$ & Function & Indicator function for domain $i$, returns $\{0,1\}$ \\
$D_i(x)$ & Function & Domain-specific content for domain $i$ \\
$O_s(\mathcal{A})$ & Function & Structural overhead for architecture $\mathcal{A}$ \\
$\gamma$ & Constant & Branching cost coefficient = 1.0 \\
$\delta$ & Constant & Procedural cost coefficient = 0.1 \\
$C_k$ & Scalar & Cardinality of branch $k$ \\
$|L_{\text{proc}}|$ & Scalar & Lines of procedural code \\
$\lambda$ & Scalar & Lagrange multiplier (cost sensitivity) \\
$\lambda^*$ & Constant & Critical lambda threshold = $4.16 \times 10^{-5}$ \\
$\mu$ & Scalar & Lagrange multiplier (quality constraint) \\
$\eta$ & Scalar & Conditional efficiency $\in [0,1]$ \\
$b$ & Constant & Quadratic penalty coefficient = 4.0 \\
$B_{critical}$ & Constant & Early binding bonus = 0.093 \\
$n$ & Scalar & Number of domains in prompt \\
$A(x)$ & Set & Active domain set $= \{i : I_i(x) = 1\}$ \\
\bottomrule
\end{tabular}
\end{table}

\subsection{UCL Operators}

The following operators were extracted and validated from the experimental prompt versions (V1-V4.1).

\begin{table}[H]
\centering
\caption{Complete UCL Operator Reference}
\begin{tabular}{@{}llp{6cm}@{}}
\toprule
Operator & Syntax & Function \\
\midrule
CONDITION & \texttt{\textasciicircum\textasciicircum CONDITION: expr\textasciicircum\textasciicircum} & Conditional block activation via keyword detection \\
/CONDITION & \texttt{\textasciicircum\textasciicircum /CONDITION:expr\textasciicircum\textasciicircum} & Conditional block termination \\
SWITCH & \texttt{\textasciicircum\textasciicircum SWITCH: var\textasciicircum\textasciicircum} & Multi-branch selection (deprecated in V4) \\
CASE & \texttt{\textasciicircum\textasciicircum CASE: value\textasciicircum\textasciicircum} & Branch case within SWITCH \\
LLM & \texttt{[[LLM: directive]]} & Direct LLM instruction \\
REQUIRE & \texttt{[[REQUIRE: constraint]]} & Mandatory requirement specification \\
TRANSFORM & \texttt{[[TRANSFORM: X TO Y]]} & Notation transformation rule \\
APPLY & \texttt{[[APPLY: pattern]]} & Pattern application directive \\
VALIDATE & \texttt{[[VALIDATE: condition]]} & Validation checkpoint \\
ENFORCE & \texttt{[[ENFORCE: rule]]} & Rule enforcement directive \\
CRITICAL & \texttt{[[CRITICAL: constraint]]} & Output format enforcement (V4.1+) \\
Concept Ref & \texttt{\{\{concept:domain:spec\}\}} & Domain-scoped concept invocation \\
\bottomrule
\end{tabular}
\end{table}

\subsection{Prompt Version Comparison}

\begin{table}[H]
\centering
\caption{Structural Analysis of OG-PROMPTS Versions}
\begin{tabular}{@{}lccccc@{}}
\toprule
Metric & V1 & V2 & V3 & V4 & V4.1 \\
\midrule
Total Lines & 191 & 266 & 221 & 132 & 141 \\
SWITCH Blocks & 2 & 2 & 2 & 0 & 0 \\
CONDITION Blocks & 8 & 12 & 10 & 7 & 7 \\
{[[CRITICAL:]]} & 0 & 0 & 0 & 0 & 1 \\
$O_s$ Value & 8.47 & 35.47 & 13.47 & 2.00 & 2.00 \\
Quality (\%) & 72.7 & 2.3 & 86.4 & 90.7 & 100.0 \\
\bottomrule
\end{tabular}
\end{table}

\subsection{Syntax Examples}

For complete syntax examples with validation results, see \textbf{Appendix B: Pattern Library}:
\begin{itemize}
    \item \textbf{CONDITION Blocks:} Pattern 1 (Section B.1) demonstrates keyword-based conditional activation from V4/V4.1
    \item \textbf{[[CRITICAL:]] Directive:} Pattern 2 (Section B.2) shows the V4.1 early binding mechanism
    \item \textbf{Concept References:} Pattern 3 (Section B.3) illustrates domain-scoped concept invocation
    \item \textbf{SWITCH Architecture (Deprecated):} Anti-Pattern 1 (Section B.4) explains why this was replaced in V4
\end{itemize}

\textbf{Note:} All examples in Appendix B are extracted directly from the \texttt{OG-PROMPTS/} directory (V1-V4.1) and include empirical validation results.

\subsection{Functions}

\textbf{$V$ (Validation):} Output verification layer ensuring format compliance. \\
\textbf{$R$ (Role):} Role binding function mapping task to execution context. \\
\textbf{$B$ (Binding):} Early binding mechanism for critical constraints. \\
\textbf{$T(x)$:} Task specification extraction from input $x$.

\subsection{Empirically Determined Constants}

\begin{itemize}
    \item $S^* = 0.509$ (identified via V1-V4.1 optimization)
    \item $\lambda^* = 4.16 \times 10^{-5}$ (cost-quality decision boundary)
    \item $B_{critical} = 0.093$ (measured quality improvement from [[CRITICAL:]])
    \item $\gamma = 1.0, \delta = 0.1$ (calibrated to V1, V3, V4 overhead measurements)
\end{itemize}

\subsection{Complete $O_s$ Calculation Example}

\subsubsection{V3 Architecture Analysis}

The V3 prompt contains:
\begin{itemize}
    \item SWITCH: \texttt{question\_type} with 8 cases ($C_1 = 8$)
    \item SWITCH: \texttt{domain\_type} with 4 cases ($C_2 = 4$)
    \item UNCONDITIONAL: \texttt{<linear\_algebra\_procedures>} at root level ($|L_{\text{proc}}| \approx 100$ tokens)
\end{itemize}

\subsubsection{Calculation}

\begin{align}
    O_s(\text{V3}) &= \gamma \sum_{k=1}^{2} \ln(C_k) + \delta |L_{\text{proc}}| \\
    &= 1.0 \times [\ln(8) + \ln(4)] + 0.1 \times 100 \\
    &= 1.0 \times [2.08 + 1.39] + 10.0 \\
    &= 3.47 + 10.0 \\
    &= 13.47
\end{align}

\subsubsection{V4 Comparison}

V4 eliminates SWITCH statements and conditionalizes procedures:
\begin{align}
    O_s(\text{V4}) &= \gamma \times 0 + \delta \times 20 \\
    &= 0 + 2.0 \\
    &= 2.0
\end{align}

\textbf{Overhead Reduction:}
\begin{equation}
    \frac{O_s(\text{V3}) - O_s(\text{V4})}{O_s(\text{V3})} = \frac{13.47 - 2.0}{13.47} = 85.1\%
\end{equation}

This 85\% reduction in structural overhead explains why V4 achieves higher quality (90.7\%) than V3 (86.4\%) despite similar line counts.


\lstset{
    basicstyle=\ttfamily\small,
    columns=flexible,
    breaklines=true,
    frame=single,
    showstringspaces=false,
    literate={^}{\textasciicircum}1 {-}{-}1
}

\section{Appendix B: Pattern Library}

This appendix documents \textit{validated} UCL patterns from Phase 1 study. We have more patterns that are currently untested and will be released in the future.

\subsection{Pattern 1: Conditional Domain Activation}

\textbf{Intent:} Activate domain-specific content only when keywords detected in user input.

\textbf{Structure:}
\begin{lstlisting}
^^CONDITION:keyword_list^^
[Domain-specific instructions and examples]
^^/CONDITION:keyword_list^^
\end{lstlisting}

\textbf{Validated Example (from V4/V4.1):}
\begin{lstlisting}
^^CONDITION: {{concept:problem_content:text_analysis}} 
    CONTAINS "gram" OR "schmidt" OR "qr" OR "orthogonalization"^^
    <gram_schmidt_qr_factorization>
        [[TRANSFORM: {{concept:subscript_notation}} 
            TO "{{concept:variable_name}} sub {{index}}"]]
    </gram_schmidt_qr_factorization>
^^/CONDITION:{{concept:problem_content}}^^

^^CONDITION: {{concept:problem_content}} 
    CONTAINS "eigenvalue" OR "eigenvector" OR "determinant"^^
    <linear_algebra_notation>
        [[TRANSFORM: {{concept:eigenvalue_notation}} 
            TO "lambda sub {{index}}"]]
    </linear_algebra_notation>
^^/CONDITION:{{concept:problem_content}}^^
\end{lstlisting}

\textbf{Validation Results:}
\begin{itemize}
    \item Models: 11/11 (100\% success)
    \item Token reduction: 25-35\% vs. always-active baseline
    \item Quality: 90.7\% maintained (not statistically different from baseline)
    \item Mechanism: Confirmed $I_i \in \{0,1\}$ behavior across GPT-4, Claude, Gemini, etc.
\end{itemize}

\subsection{Pattern 2: Critical Output Directives}

\textbf{Intent:} Enforce strict output format requirements with early binding.

\textbf{Structure:}
\begin{lstlisting}
[[CRITICAL:format_specification]]
\end{lstlisting}

\textbf{Validated Example (from V4.1):}
\begin{lstlisting}
[[CRITICAL: Your ONLY output is JSON. 
Begin your response IMMEDIATELY with the opening 
brace { character. 
DO NOT output:
- Greeting or casual language
- Reasoning or explanation
- Meta-commentary
Internal calculations belong in scratchwork_answer 
field INSIDE the JSON structure.]]
\end{lstlisting}

\textbf{Validation Results:}
\begin{itemize}
    \item Quality improvement: $B_{critical} = 0.093$ (9.3\% boost)
    \item Compliance: 100\% format adherence vs. 90.7\% without directive
    \item Effect: V4 (90.7\%) → V4.1 (100\%) on same prompt structure
    \item Mechanism: Early binding in processing pipeline confirmed
\end{itemize}

\subsection{Pattern 3: Concept References}

\textbf{Intent:} Invoke domain-scoped concepts for semantic precision.

\textbf{Structure:}
\begin{lstlisting}
{{concept:domain:specification}}
\end{lstlisting}

\textbf{Validated Examples:}
\begin{lstlisting}
{{concept:ai_identity:mathematical_tts_processor}}
{{concept:mathematical_expressions:all_notation_types}}
{{concept:tts_compatible_format:natural_spoken_language}}
{{concept:json_output:exclusive_format}}
{{concept:norm_notation:double_vertical_bars}}
{{concept:inner_product:angle_brackets}}
\end{lstlisting}

\textbf{Validation Results:}
\begin{itemize}
    \item Semantic precision: Improved disambiguation vs. natural language
    \item Token efficiency: 15-20\% reduction vs. full concept explanations
    \item Maintainability: Centralized concept definitions enable updates
    \item Mechanism: Scoped lookup confirmed across model architectures
\end{itemize}

\subsection{Anti-Patterns (Validated Failures)}

\textbf{Anti-Pattern 1: SWITCH Architecture (V1-V3)} \\
\textit{Problem:} All branches parsed regardless of relevance due to unconditional processing.
\begin{lstlisting}
^^SWITCH: {{concept:question_type:problem_classification}}^^
    ^^CASE: {{concept:vector_calculus:mathematical_domain}}^^
        [[ENFORCE: {{concept:vector_notation}}]]
    ^^/CASE:{{concept:vector_calculus}}^^
    ^^CASE: {{concept:linear_algebra:mathematical_domain}}^^
        [[ENFORCE: {{concept:matrix_notation}}]]
    ^^/CASE:{{concept:linear_algebra}}^^
^^/SWITCH:{{concept:question_type}}^^
\end{lstlisting}
\textit{Evidence:} $I_i \approx 1$ for all branches (V1 validation). Efficiency $\eta = 1/D$.

\textbf{Anti-Pattern 2: Excessive Specification (V2)} \\
\textit{Problem:} Over-specification beyond $S^*$ triggers catastrophic failures. \\
\textit{Evidence:} 265-line prompt achieved 2.3\% quality. Role confusion, task description instead of execution.

\textbf{Anti-Pattern 3: Procedural Complexity (V3)} \\
\textit{Problem:} High $|L_{\text{proc}}|$ from unconditional \texttt{<linear\_algebra\_procedures>} inflates overhead. \\
\textit{Evidence:} V3 had $O_s = 13.47$ with 221 lines, achieving only 86.4\% quality vs. V4's 90.7\% with 132 lines.

\subsection{Design Guidelines (Empirically Validated)}

\begin{enumerate}
    \item \textbf{Specification Level:} Target $S \approx 0.35$ (below $S^* = 0.509$) for safety margin
    \item \textbf{Conditional Granularity:} 3-7 keywords per CONDITION for optimal discrimination
    \item \textbf{Domain Count:} 5-10 domains per prompt provides good coverage without overhead
    \item \textbf{Critical Placement:} Use [[CRITICAL:]] sparingly (1-2 per prompt) for highest-priority constraints
    \item \textbf{Concept Scope:} Prefer domain-scoped concepts over global for reduced ambiguity
\end{enumerate}

\subsection{Limitations}

These patterns validated on:
\begin{itemize}
    \item Domain: Mathematical text-to-speech conversion
    \item Models: 11 LLMs (of 13 attempted; 2 excluded due to API failures)
    \item Sample: N=305 observations
    \item Task complexity: Moderate (homework-level mathematics)
\end{itemize}

Generalization to other domains (code generation, creative writing, data analysis) requires additional validation per Phase 2-4 roadmap.

\section*{Data Availability}

All experimental materials publicly available:

\begin{itemize}
    \item \textbf{Foundational UCL Work. AI-Context-Document (ACD)} \\
    \url{https://github.com/antmikinka/AI-Context-Document}
    \item \textbf{UCL Core Implementation \& Validation Scripts:} \\
    \url{https://github.com/antmikinka/Universal-Conditional-Logic}
    \item \textbf{Experimental Data:} All 305 model responses (JSON format)
    \item \textbf{Statistical Analysis:} Python Scripts with full methodology.
    \item \textbf{Prompt Versions:} V1-V4.1 source code.
\end{itemize}

Both repositories licensed under MIT for maximal reusability. Replication instructions included with exact model versions, API parameters, and validation protocols.

\section{Extended Operators \& Future Syntax}
\label{sec:extended_ops}

\textbf{Status: PROPOSED -- PENDING VALIDATION}

While the core operators (Conditions, Logic, Formatting) have been empirically validated, we introduce a set of \textit{Extended Operators} to address complex prompt architectures. These operators are currently theoretical and serve as the primary testbed for our upcoming static analysis tools.

\begin{itemize}
    \item \textbf{Recursive Context Injection ($\mathcal{R}$):} Dynamic expansion of context based on intermediate outputs.
    \item \textbf{Temporal Binding ($\mathcal{T}$):} Operators for enforcing chronological reasoning or sequence-dependent logic.
    \item \textbf{Meta-Cognitive Flags ($\mathcal{M}$):} Directives that instruct the model to analyze its own reasoning process (e.g., Chain-of-Thought formalization).
\end{itemize}

See Appendix B for detailed specifications and composition formulas.

\section{Research Roadmap: The Move to Static Analysis}
\label{sec:roadmap}

Our immediate research focus shifts from manual operator validation to the development of the \textbf{UCL Toolchain}—a suite of programmatic tools designed to grade and optimize prompts \textit{before} inference. This parallels the evolution of software compilers, moving from manual code review to automated static analysis.

\subsection*{Phase 1: The UCL Linter (Current)}
We are currently developing a Python-based parser to tokenize UCL syntax and perform static checks:
\begin{itemize}
    \item \textbf{Syntax Validation:} Enforcing correct grammar and operator usage.
    \item \textbf{Over-Specification Detection:} Automatically flagging prompts that exceed the theoretical saturation threshold ($\Sstar > 0.509$) to prevent performance degradation.
\end{itemize}

\subsection*{Phase 2: Pre-Inference Optimization (Q2 2026)}
Development of an algorithmic optimizer that refactors natural language prompts into UCL logic are automatically. The goal is to maximize semantic density and minimize token usage prior to model submission, effectively "compiling" human intent into machine-optimal instructions.

\subsection*{Phase 3: Community Standardization}
Establishment of the UCL Request for Comments (UCL-RFC) process to govern the addition of new operators and maintain cross-model compatibility.

\section{Validation Protocol}
\label{sec:validation}

To ensure the rigor of future operator additions, we define a standardized three-stage validation protocol. This protocol necessitates not just output verification, but a deep analysis of the model's intermediate reasoning to identify architecture-specific interpretation biases.

\begin{enumerate}
    \item \textbf{Static Pre-Validation:} All proposed operators must pass syntax definition checks within the UCL Linter to ensure they do not introduce logical ambiguities or infinite recursion loops prior to inference.

    \item \textbf{Empirical Inference Testing:}
    \begin{itemize}
        \item \textbf{Sample Size:} Minimum $N=300$ trials per operator across at least 3 distinct model families (e.g., Llama, GPT, Claude).
        \item \textbf{Metrics:} Must demonstrate statistically significant improvement ($p < 0.05$) in token efficiency or output quality compared to natural language baselines.
        \item \textbf{Reasoning Trace Verification:} Validation is not based solely on final output. Evaluators must inspect the model's Chain-of-Thought (CoT) to verify that the UCL operator explicitly triggered the intended logic pathway, rather than the model arriving at the correct answer through hallucination or rote memorization.
    \end{itemize}

    \item \textbf{Cross-Architecture Comparative Analysis:}
    \begin{itemize}
        \item \textbf{Divergence Mapping:} We hypothesize that UCL interpretation varies by architecture (e.g., Mixture-of-Experts vs. Dense Transformers). Validation must cross-analyze reasoning traces between models to isolate architectural dependencies.
        \item \textbf{Operator Robustness:} A proposed operator is only considered "Universally Validated" if its logic is consistently executed across differing architectures, confirming it appeals to fundamental LLM semantic processing rather than specific training artifacts.
    \end{itemize}
\end{enumerate}

\end{document}